\documentclass{article}

\usepackage[preprint,nonatbib]{neurips_2019}

\usepackage[utf8]{inputenc} 
\usepackage[T1]{fontenc}    
\usepackage{url}            
\usepackage{booktabs}       
\usepackage{amsfonts}       
\usepackage{nicefrac}       
\usepackage{microtype}      
\usepackage{graphicx}
\usepackage{amsmath}
\usepackage{amsthm}

\usepackage{relsize}
\usepackage{url}
\usepackage{tabularx, booktabs}
\usepackage{caption}
\usepackage{pifont}
\usepackage{pgf,tikz}

\usepackage{tabulary,multirow,overpic}

\definecolor{citecolor}{RGB}{34,139,34}

\usepackage[breaklinks=true,letterpaper=true,colorlinks, citecolor=citecolor,bookmarks=false]{hyperref}

\newcolumntype{x}[1]{>{\centering\arraybackslash}p{#1pt}}

\newcommand{\app}{\raise.17ex\hbox{$\scriptstyle\sim$}}

\newlength\savewidth\newcommand\shline{\noalign{\global\savewidth\arrayrulewidth
  \global\arrayrulewidth 1pt}\hline\noalign{\global\arrayrulewidth\savewidth}}
  
\newcommand{\tablestyle}[2]{\setlength{\tabcolsep}{#1}\renewcommand{\arraystretch}{#2}\centering\footnotesize}

\theoremstyle{definition}
\newtheorem{definition}{Definition}
\theoremstyle{claim}
\newtheorem{claim}{Claim}
\usepackage{amsthm}

    \newtheoremstyle{TheoremNum}
        {\topsep}{\topsep}              
        {\itshape}                      
        {}                              
        {\bfseries}                     
        {.}                             
        { }                             
        {\thmname{#1}\thmnote{ \bfseries #3}}
    \theoremstyle{TheoremNum}
    \newtheorem{theorem}{Theorem}

\usepackage[maxbibnames=10]{biblatex}
\title{Rethinking preventing class-collapsing in metric learning with margin-based losses }

\author{
Elad Levi$^{1}$, \,\,
Tete Xiao$^{2}$, \,\,
Xiaolong Wang $^{3}$, \,\,
Trevor Darrell$^{1,2}$
\vspace{3pt}\\

$^1$Nexar, $^2$UC Berkeley, $^3$UC San Diego\\
}

\begin{document}



 

\maketitle

\begin{abstract}
\label{sec:abstract}
Metric learning seeks perceptual embeddings where visually similar instances are close and dissimilar instances are apart, but learned representations can be sub-optimal when the distribution of intra-class samples is diverse and distinct sub-clusters are present. Although theoretically with optimal assumptions, margin-based losses such as the triplet loss and margin loss have a diverse family of solutions.
We theoretically prove and empirically show that under reasonable noise assumptions, margin-based losses tend to project all samples of a class with various modes onto a single point in the embedding space, resulting in class collapse that usually renders the space ill-sorted for classification or retrieval. 
To address this problem, we propose a simple modification to the embedding losses such that each sample selects its nearest same-class counterpart in a batch as the positive element in the tuple. This allows for the presence of multiple sub-clusters within each class.
The adaptation can be integrated into a wide range of metric learning losses.
 Our method demonstrates clear benefits on various fine-grained image retrieval datasets over a variety of existing losses; qualitative retrieval results show that samples with similar visual patterns are indeed closer in the embedding space. 
\vspace{-15pt}
\end{abstract}


\section{Introduction}
\label{sec:introduction}

Metric learning aims to learn an embedding function to lower dimensional space, in which semantic similarity translates to neighborhood relations in the embedding space~\cite{Lowe1995}. Deep metric learning approaches achieve promising results in a large variety of tasks such as face identification ~\cite{chopra2005learning,taigman2014deepface,NIPS2014_5416},  zero-shot learning ~\cite{frome2013devise}, image retrieval ~\cite{hoffer2015deep,gordo2016deep} and fine-grained recognition~\cite{wang2014learning}. 

In this work we investigate the family of losses which optimize for an embedding representation that enforces that all modes of intra-class appearance variation project to a single point in embedding space. Learning such an embedding is very challenging when classes have a diverse appearance. This happens especially in real-world scenarios where the class consists of multiple modes with diverse visual appearance. Pushing all these modes to a single point in the embedding space requires the network to memorize the relations between the different class modes, which could reduce the generalization capabilities of the network and result in sub-par performance.

Recently researchers observed that this phenomena, where all modes of class appearance ``collapse'' to the same center, occurs in case of the classification SoftMax loss~\cite{Qian2019}. They proposed a multi-center approach, where multiple centers for each class are used with the SoftMax loss to capture the hidden distribution of the data to solve this issue. 
In the metric learning field, a positive sampling method has been proposed~\cite{DBLP:conf/wacv/XuanSP20} with respect to the N-pair loss \cite{NIPS2016_6200}  in order to relax the constraints on the intra-class relations. For margin-based losses such as the triplet loss~\cite{chechik2010large} and margin loss~\cite{Wu2017}, it was believed that they might offer some relief from class collapsing~\cite{wang2014learning,Wu2017}. From a theoretic perspective, we prove that with optimal assumptions on the hypothesis space and the training procedure, it is indeed true that the margin-based losses have a minimal solutions without class collapsing. However, we formulate a noisy framework and prove that with modest noise assumptions on the labels, the margin-based losses yet suffer from class collapse and the easy positive sampling method proposed in~\cite{DBLP:conf/wacv/XuanSP20} allow more diverse solutions. Adding noise to the labels allow modelling both the aleatoric and the approximation uncertainties of the neural network, therefore it batters represent the training process on real-world datasets with fixed restricted network architecture.  

We complement our theoretical study with an extensive empirical study, which demonstrates the class-collapsing phenomena on real-world datasets, and show that the easy positive sampling method is able to create a more diverse embedding which results in better generalization performances. These findings suggest that the noisy environment framework better fits the training dynamic of neural networks in real-world use cases.

\begin{figure}[t!]
	\begin{center}
        \includegraphics[width=1\linewidth]{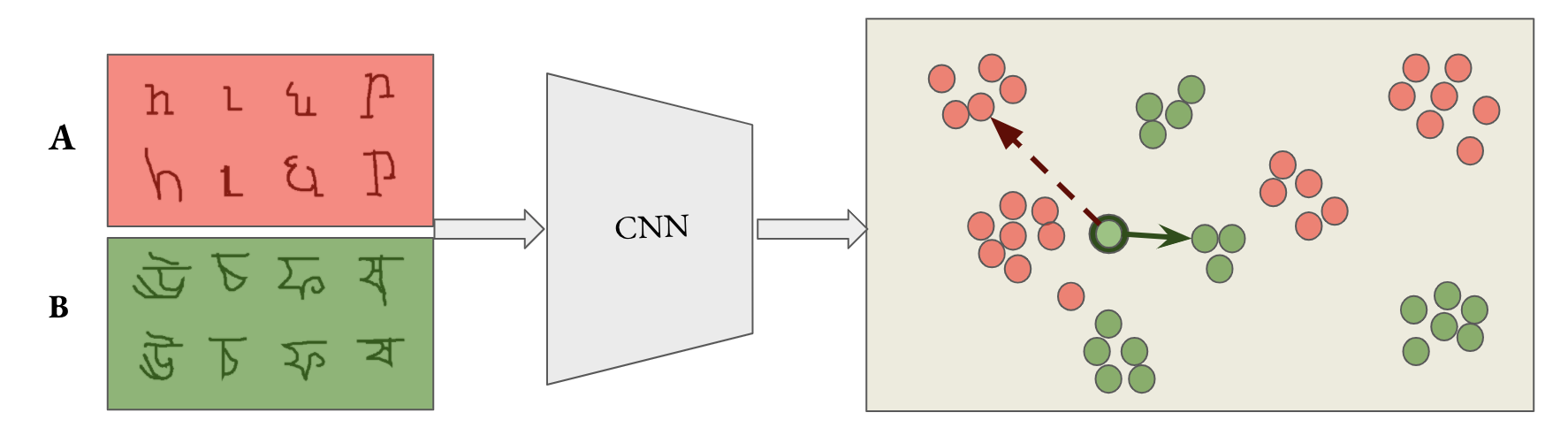}
	\caption{\small{Given an anchor (circle with dark ring), our approach samples the closest positive example in the embedding space as the positive element. This results in pushing the anchor only towards the closest element direction (green arrow), which allows the embedding to have multiple clusters for each class.}}
	\label{fig:overview}
	\end{center}
\end{figure}

\section{Related Work}
\label{sec:related_work}

\textbf{Sampling methods.} Designing a good sampling strategy is a key element in deep metric learning. Researchers have been proposed sampling methods when sampling both the negative examples as well as the positive pairs. For negative samples, studies have focused on sampling hard negatives to make training more efficient~\cite{simo2015discriminative,Schroff2015,Wang_UnsupICCV2015,oh2016deep,parkhi2015deep}. Recently, it has been shown that increasing the negative examples in training can significantly help unsupervised representation learning with contrastive losses~\cite{he2019momentum,wu2018unsupervised,chen2020simple}.  Besides negative examples, methods for sampling hard positive examples have been developed in classification and detection tasks~\cite{loshchilov2015online,shrivastava2016training,arandjelovic2016netvlad,cubuk2019autoaugment,singh2017hide,Wang_afrcnnCVPR2017}. The central idea is to perform better augmentation to improve the generalization in testing~\cite{cubuk2019autoaugment}. In contrast,  Arandjelovic et al.~\cite{arandjelovic2016netvlad}, proposed to perform a positive sampling by assigning the near instance from the same class as the positive instance. As the positive training set is highly noisy in their setting, this method leads to features invariant to different perspectives. Different from this approach, we use this method in a clean setting, where the purpose is to get the opposite result of maintaining the inner-class modalities in the embedding space. Using easy positive sampling has been also proposed with respect to the N-pair loss~\cite{DBLP:conf/wacv/XuanSP20} in order to relax the constraints of the loss on the intra-class relations. From a theoretic perspective, we prove that in a clean setting this relaxation is redundant for other popular metric losses like the triplet loss~\cite{chechik2010large} and margin loss~\cite{Wu2017}. We formulate the noisy-environment setting and prove that in this case the triplet and margin losses also suffer from class-collapsing and using an easy positive sampling method optimizes for solutions without class-collapsing. We also provide an empirical study that supports the theoretic analysis.

\textbf{Model uncertainty.} 
There are three types of sources for uncertainty: epistemic, aleatoric and approximation \cite{2cfc44fb6aa842f5ac03e51409667171}. 
The epistemic uncertainty describes the lack of knowledge of the model, the approximation uncertainty describes the model limitation to fit the data, and the aleatoric uncertainty accounts for the stochastic nature of the data.
While the epistemic uncertainty is relevant only in regions of the feature space where there is a lack of data, both the approximation and the aleatoric uncertainties are relevant also in regions where there is labelled data.
In this work, we model the approximation and the aleatoric uncertainties, by adding noise to the labels. This noise can stand for the data stochasticity in the aleatoric uncertainty case, or the results of the Bayes optimal model within the hypothesis space in case of the approximation uncertainty. The approximation uncertainty in deep neural networks is considered to be negligible \cite{citeulike:3561150}. However, we prove that even a small amount of noise results in a degeneration of the family of optimal solutions in case of the margin-based losses.

Learning with noisy labels is a practical problem when applied to the real world~\cite{scott2013classification,natarajan2013learning,pmlr-v97-shen19e,reed2014training,jiang2017mentornet,khetan2017learning,malach2017decoupling}, especially when training with large-scale data~\cite{sun2017revisiting}. One line of work applies a data-driven curriculum learning approach where the data that are most likely labelled correctly are used for learning in the beginning, and then harder data is taken into learning during a later phase~\cite{jiang2017mentornet}. Researchers have also tried on to apply the loss only on the easiest top k-elements in the batch, determined by the lowest current loss~\cite{pmlr-v97-shen19e}. Inspired by these the easy positive sampling method focuses on selecting only the top easiest positive relations in the batch.


\textbf{Beyond memorization.} Deep networks are shown to be extremely easy to memorize and over-fit to the training data~\cite{zhang2016understanding,recht2018cifar,recht2019imagenet}. For example, it is shown the network can be trained with randomly assigned labels on the ImageNet data, and obtain $100\%$ training accuracy if augmentations are not adopted. Moreover, even the CIFAR-10 classifier performs well in the validation set, it is shown that it does not really generalize to new collected data which is visually similar to the training and validation set~\cite{recht2018cifar}. In this paper, we show that when allowing the network the freedom not to have to learn inner-class relation between different class modes, we can achieve much better generalization, and the representation can be applied in a zero-shot setting. 

\section{Preliminaries}
\label{gen_inst}
Let $X=\{x_1,..,x_n\}$ be a set of samples with labels $y_i \in{\{1,..,m\}} $. The objective of metric learning is to learn an embedding $f(\cdot,\theta)\xrightarrow{}\mathbb{R}^k$, in which the neighbourhood of each sample in the embedding space contains samples only from the same class. One of the common approaches for metric learning is using embedding losses in which at each iteration, samples from the same class and samples from different classes are chosen according to same sampling heuristic. The objective of the loss is to push away projections of samples from different classes, and pull closer projections of samples from a same class. In this section, we introduce a few popular embedding losses.

\textbf{Notation:} Let $x_i,x_j\in{X}$, define: $D^f_{x_i,x_j} = \Arrowvert\ f(x_i) - f(x_j) \Arrowvert\ ^2 $. In cases where there is no ambiguity we omit $f$ and simply write $D_{x_i,x_j}$ .  We also define the function $\delta_{x_{i},x_{j}}=\begin{cases}
1 & \mathrm{if}~~y_{i}=y_{j}\\
0 & \mathrm{otherwise}
\end{cases}$. 
Lastly, for every $a \in \mathbb{R}$, denote $(a)_+ := max(a,0)$. \par

The Contrastive loss~\cite{Hadsell} takes tuples of samples embeddings. It pushes tuples of samples from different classes apart and pulls tuples of samples from the same class together.
\[\mathcal{L}_{con}^{f}(x_i,x_j) = \delta_{x_{i},x_{j}}\cdot D^f_{x_i,x_j} + (1-\delta_{x_{i},x_{j}})\cdot (\alpha- D^f_{x_i,x_j})_+\]
Here $\alpha$ is the margin parameter which defines the desired minimal distance between samples from different classes.   

While the Contrastive loss imposes a constraint on a pair of samples, the Triplet loss~\cite{chechik2010large} functions on a triplet of samples. Given a triplet $x_a,x_p,x_n\in X$, let  \[h^{f}(x_a,x_p,x_n)=(D^f_{x_a,x_p} - D^f_{x_p,x_n} + \alpha)_+\] the triplet loss is defined by 
\[\mathcal{L}_{trip}^f(x_a,x_p,x_n) = \delta_{x_{a},x_{p}}\cdot(1-\delta_{x_{a},x_{n}})\cdot h^{f}(x_a,x_p,x_n)\]

The Margin loss~\cite{Wu2017} aims to exploit the flexibility of Triplet loss while maintaining the computational efficiency of the Contrastive loss. This is done by adding a variable ($\beta_x$ for $x \in X$) which determines the boundary between positive and negative pairs; given an anchor $t\in X$, let \[g(z_1,z_2)=( z_1 - z_2 + \alpha)_+\] the loss is defined by \[ \mathcal{L}^{f,\beta}_{mar}(t,x) = \delta_{t,x}\cdot g(D^f_{t,x} , \beta_{t}) + (1-\delta_{t,x})\cdot g(\beta_{t}, D^f_{t,x})\]



\section{Class-collapsing}
\label{headings}

The contrastive loss objective is to pull all the samples with the same class to a single point in the embedding space. We call this the \textit{Class-collapsing} property. Formally, an embedding $f: X \xrightarrow{} \mathbb{R}^m$ has the class-collapsing property, if there exists a label $y$ and a point $p\in \mathbb{R}^m$ such that $\{f(x_i) |\quad y_i = y \} = \{p\}$.

\subsection{Embedding losses optimal solution}
It is easy to see that an embedding function $f$ that minimizes: \[\mathbb{O}_{con}(f)=\frac{1}{n^{2}}\left(\sum_{x_{i},x_{j}\in X}\mathcal{L}_{con}^{f}(x_{i},x_{j})\right)\]
has the class-collapsing property with respect to all classes. However, this is not necessarily true for the Triplet loss and the Margin loss.


For simplification for the rest of this subsection we will assume there are only two classes. Let $A\subset X$ be a subset of elements such that all the elements in $A$ belongs to one class and all the element in $A^c$ belong to the other class.

Recall some basic set definitions.
\theoremstyle{definition}
\begin{definition}{}
For all sets $Y,Z \subset \mathbb{R}^m$ define:
\begin{enumerate}
  \item The diameter of $Y$ is defined by: \[diam(Y) = \sup\{\Arrowvert y-z \Arrowvert\ | y,z\in Y\}\] 
  \item The distance between Y and Z is: \[\Arrowvert Y- Z \Arrowvert = \inf\{\Arrowvert y-z \Arrowvert\ | y\in Y, z\in Z\} \]
\end{enumerate}
\end{definition}

It is easy to see that if $f: X \xrightarrow{}\mathbb{R}^m$ is an embedding, such that $diam(f(A)) < 2 \cdot \alpha +  \Arrowvert f(A) - f(B) \Arrowvert$, then: \[\mathbb{O}_{trip}(f)=\frac{1}{n^{3}}\left(\sum_{x_{i},x_{j},x_{k}\in X}\mathcal{L}_{trip}^{f}(x_{i},x_{j},x_{k})\right) = 0.\]\par
Moreover, fixing $\beta_{x_i} =\alpha$ for every $x_i \in X$, then:  

\[\mathbb{O}_{margin}(f,\beta)=\frac{1}{n^{2}}\left(\sum_{x_{i},x_{j}\in X}\mathcal{L}^{f,\beta}_{margin}(x_{i},x_{j})\right) = 0.\] 

It can be seen that indeed, the family of embeddings which induce the global-minimum with respect to the Triplet loss and the Margin loss, is rich and diverse. However, as we will prove in the next subsection, this does not remain true in a noisy environment scenario.

\subsection{Noisy environment analysis}
For simplicity we will also discuss in this section the binary case of two labels, however this could be extended easily to the multi-label case.

The noisy environment scenario can be formulated by adding uncertainty to the label class.
More formally, let $Y = \{Y_1,..,Y_n\}$ be a set of independent binary random variables. Let $A_1,..,A_t\subset X$, $0.5<p<1$ such that: $|A_j| = \frac{n}{t}$ and \[\mathbb{P}(Y_i=k)= \begin{cases}
p & x_i\in A_k\\
q':=\frac{1-p}{t-1} & x_i\notin A_k
\end{cases} \]

We can also reformulate $\delta$ as a binary random variable such that:
\[ \overline{\delta}_{Y_{i},Y_{j}}:= \mathbf{1}_{Y_{i}=Y_{j}}  \]

The loss with respect to embedding $f$ is a random variable and the objective is to minimize its expectation

\[\mathbb{E}\mathcal{L}_{tr}^{f}(x_{i},x_{j},x_{k})=\mathbb{E}\left(\bar{\delta}_{Y_{i},Y_{j}}\cdot(1-\bar{\delta}_{Y_{i},Y_{k}})\right)\cdot h^f(x_i,x_j,x_k)\]
Therefore, we are searching for an embedding function which minimize 
\[ \mathbb{E}\mathbb{O}_{trip}(f)=\frac{1}{n^{3}}\sum_{x_{i},x_{j},x_{k}\in X}\mathbb{E}\mathcal{L}_{trip}^{f}(x_{i},x_{j},x_{k}) \]
In Appendix A we will prove the following two theorems.

\begin{theorem}[1] 
\label{Lagrange}
Let $f:X \xrightarrow{}\mathbb{R}^m$ be an embedding, which minimize $ \mathbb{E}\mathbb{O}_{trip}(f)$, then $f$ has the class-collapsing property with respect to all classes.
\end{theorem}

Similarly, we can define:
 \begin{align*}
 & \mathbb{E}\mathcal{L}_{mar}^{f}(x_{i},x_{j})=\mathbb{E}\bar{\delta}_{Y_{i},Y_{j}}\cdot(D_{x_{i},x_{j}}^{f}-\beta_{x_{i}}+\alpha)_{+}\\
 & +\mathbb{E}(1-\bar{\delta}_{Y_{i},Y_{j}})\cdot(\beta_{x_{i}}-D_{x_{i},x_{j}}^{f}+\alpha)_{+}
\end{align*}

\begin{theorem}[2]
\label{Lagrange}
Let $f: OX \xrightarrow{}\mathbb{R}^m$ be an embedding, which minimize \[\mathbb{E}\mathbb{O}_{margin}(f,\beta)=\frac{1}{n^{2}}\sum_{x_{i},x_{j}\in X}\mathbb{E}\mathcal{L}_{margin}^{f}(x_{i},x_{j}),\] then $f$ has the class-collapsing property with respect to all classes.
\end{theorem}


In conclusion, although theoretically in clean environments the Triplet loss and Margin loss should allow more flexible embedding solutions, this does not remain true when noise is considered. On a real-world data, where mislabeling and ambiguity can be usually be found, the optimal solution with respect to both these losses becomes degenerate.

\subsection{Easy Positive Sampling (EPS)}

Using standard embedding losses for metric learning can result in an embedding space in which visually diverse samples from the same class are all concentrated in a single location in the embedding space. 
Since the standard evaluation and prediction method for image retrieval tasks are typically based on properties of the K-nearest neighbours in the embedding space, the class-collapsing property is a side-effect which is not necessarily in order to get optimal results. 
In the next section, we will show experimental results, which support the assumption that complete class-collapsing can hurt the generalization capability of the network.

To address the class-collapsing issue we propose a simple method for sampling, which results in weakening the objective penalty on the inner-class relations, by applying the loss only on the closest positive sample. Formally we define the EPS sampling in the following way; given a mini-batch with $N$ samples, for each sample $a$, let $C_a$ be the set of elements from the same class as $a$ in the mini-batch,  we choose the positive sample $p_a$ to be \[
\arg\limits\min_{t\in{C_a}}(\Arrowvert f(t)-f(a)\Arrowvert)
\]
For negative samples $n_a$ we can choose according to various options. In this paper we use the following methods: \textbf{(a)} Choosing randomly from all the elements which are not in $C_a$. \textbf{(b)} Using distance sampling~\cite{Wu2017}. \textbf{(c)} semi-hard sampling~\cite{Schroff2015},\textbf{(d)} MS hard-mining sampling~\cite{wang2019multi}.
We then apply the loss on the triplets $(a,p_a,n_a)$. Using such sampling changes the loss objective such that instead of pulling all samples in the mini-batch from the same class to be close to the anchor, it only pulls the closest sample to the anchor (with respect to the embedding space) in the mini-batch, see Figure \ref{fig:overview}.

In Appendix B, we formalize this method in the noisy environment framework. We prove (Claim 1,2) that every embedding which has the class collapsing property is \emph{not} a minimal solution with respect to both the margin and the triplet loss with the easy positive sampling. Furthermore, in Claim 3,4 we prove that the objective of the losses with EPS on tuples/triplets is to push away every element (including positive elements), that is not in the k-closest elements to the anchor, where k is determined by the noise level $p$. Therefore, if we apply the EPS method on a mini-batch which has small number of positive elements from each modality, in such case adding the EPS to the losses not only relaxes the constraints on the embedding, allowing the embedding to have multiple inner-clusters. It also optimizes the embedding to have this form.


\section{Experiments}
We test the EPS method on image retrieval and clustering datasets. We evaluate the image retrieval quality based on the recall@k metric~\cite{Jegou2011} , and the clustering quality by using the normalized mutual information score (NMI)~\cite{Manning2008}. The NMI measures the quality of clustering alignments between the clusters induced by the ground-truth labels and clusters induced by applying clustering algorithm on the embedding space. The common practice to choose the NMI clusters is by using K-means algorithm on the embedding space, with K equal to the number of classes. However, this prevents from the measurement capturing more diverse solutions in which homogeneous clusters appear only when using larger amount of clusters. Regular NMI prefers solutions with class-collapsing. Therefore, we increase the number of clusters in the NMI evaluation (denote it by NMI+) we also report the regular NMI score.   


\subsection{MNIST Even/Odd Example}

\begin{figure*}[t!]
	\begin{center}
        \includegraphics[width=\linewidth]{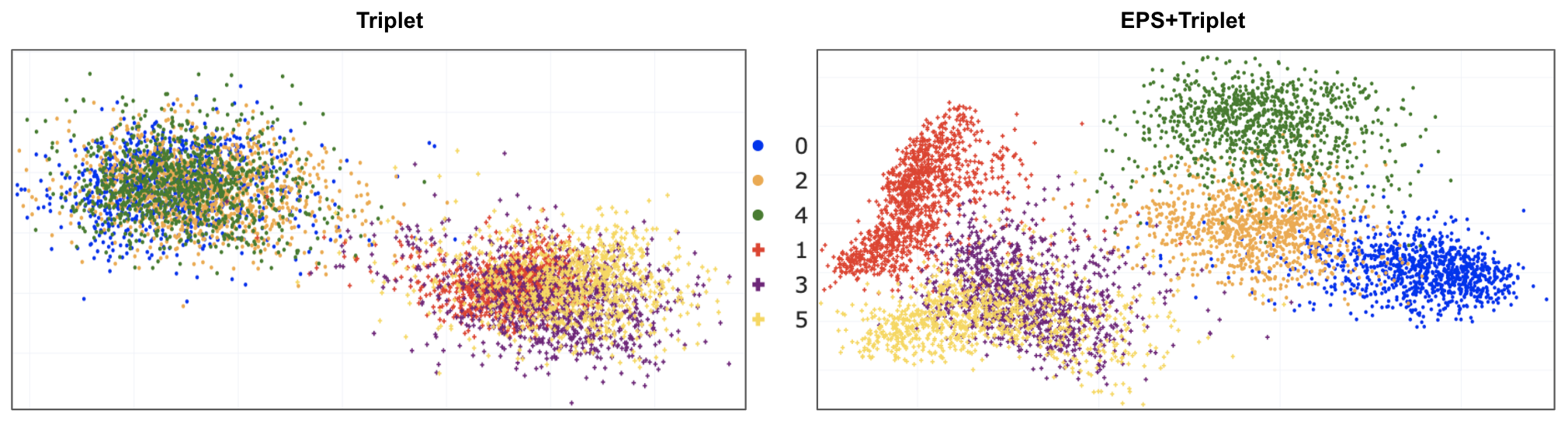}
	\caption{\small{Embedding examples from  the MNIST validation set, after training using only even/odd labels. Different colors indicate different digits. \textbf{Left:} Using Triplet-loss, class collapsing pushes  all intra-class digits to overlapping  clusters.
	\textbf{Right:} With EPS, different digits form separate clusters. 
	Retrieval or classification using the odd-vs-even task/metric is more effectively implemented using the embedding on the right, even though the embedding on the left is learned with a loss that more strictly optimizes for the task.
	}}
	\vspace{+10pt}
	\label{fig:mnist}
	\end{center}
\end{figure*}

To demonstrate the  class-collapsing phenomena, we take the MNIST dataset~\cite{Lecun1998}, and split the digits according to odd and even. From a visual perspective this is an arbitrary separation. We took the first 6 digits for training and left the remaining 4 digits for testing. We used a simple shallow architecture which results in an embedding function from the image space to $\mathbb{R}^2$ (For implementation details see Appendix C).


 We train the network using the triplet loss. We compare the EPS method to random sampling of positive examples (the regular loss).  As can be seen in Figure~\ref{fig:mnist}, the regular training without EPS suffers from class-collapsing. Training with EPS creates a richer embedding in which there is a clear separation not only between the two-classes, but also between different digits from the same class. As expected, the class-collapsing embedding preforms worse on the test data with the unseen digits, see Table \ref{tab:Mnist_table}.
 


\begin{table}[!t]

\centering

\tablestyle{3.75pt}{1.1}
\begin{tabular}{l|x{35}x{35}|x{35}x{35}|x{35}x{35}}
\multicolumn{1}{c|}{\multirow{2}{*}{}} & \multicolumn{2}{c|}{MNIST Train Digits} & \multicolumn{2}{c|}{MNIST Test Digits} \\
& Triplet & Trip+EPS & Triplet & Trip+EPS  \\
\shline
R@1 & 42.0 &\textbf{65.8} & 35.2 & \textbf{42.3}  \\
R@5 & 87.5 & \textbf{93.6} & 80.9 & \textbf{83.9}   \\
R@10 & 96.6 & \textbf{97.4} & 93.3 & \textbf{93.6}  \\

\end{tabular}

\vspace{+10pt}
\caption{Recall@k evaluated on MNIST dataset. The train classes are digits 0-5 and the test classes are digits 6-9 }
\label{tab:Mnist_table}
\end{table}

\subsection{Fine-grained Recognition Evaluation}
We compare the EPS approach to previous popular sampling methods and losses. The evaluation is conducted on standard benchmarks for zero-shot learning and image retrieval following the common splitting and evaluation practice~\cite{Wu2017,MovshovitzAttias2017,6321}. We build our implementation on top of the framework of~\cite{Roth2020}, which allows us to have a fair comparison between all the tested methods with an embedding of fix size (128). For more implementation details and consistency of the results, see Appendix C.

\subsubsection{Datasets}
We evaluate our model on the following datasets.
\begin{itemize}
  \item \textbf{Cars-196}~\cite{KrauseStarkDengFei-Fei_3DRR2013}, which contains 16,185 images of 196 car models. We follow the split in \cite{Wu2017}, using 98 classes for training and 98 classes for testing.
  \item \textbf{CUB200-2011}~\cite{WahCUB_200_2011}, which contains 11,788 images of 200 bird species. We also follow \cite{Wu2017}, using 100 classes for training and 100 for testing.
\item \textbf{Omniglot}~\cite{Lake2015}, which contains 1623 handwritten characters from 50 alphabet. In our experiments we only use the alphabets labels during the training process, i.e, all the characters from the same alphabet has the same class. We follow the split in~\cite{Lake2015} using 30 alphabets for training and 20 for testing.
\end{itemize}

\begin{table*}[!t]
\centering

\tablestyle{4.2pt}{1.05}
\begin{tabular}{l|x{22}x{22}x{22}|x{22}x{22}|x{22}x{22}x{22}|x{22}x{22}}
\multicolumn{1}{c|}{\multirow{2}{*}{model}} & \multicolumn{5}{c|}{Cars-196} & \multicolumn{5}{c}{CUB-200} \\
& R@1 & R@2 & R@4 & NMI & NMI+ & R@1 & R@2 & R@4 & NMI & NMI+ \\
\shline
Trip. + SH~\cite{Schroff2015} & 51.5 & 63.8 & 73.5 & 53.4 & - & 42.6 & 55.0 & 66.4 & 55.4 & - \\
Trip. + SH{$^\dagger$} & 76.1 & 84.4 & 90.0 & 65.1 & 68.5 & 61.5 & 73.4 & 82.5 & 66.2 & 68.1 \\
ProxyNCA~\cite{MovshovitzAttias2017} & 73.2 & 82.4 & 86.4 & 64.9 & - & 49.2 & 61.9 & 67.9 & 64.9 & - \\
ProxyNCA{$^\dagger$} & 77.1 & 85.2 & 91.2 & 65.6 & 68.9 & 63.1 & 74.8 & 83.8 & 67.2 & 68.7 \\
Dist-Margin~\cite{Wu2017} & 79.6 & 86.5 & 91.9 & \textbf{69.1} & 70.4 & 63.6 & 74.4 & 83.1 & \textbf{69.0} & 68.7 \\
MS~\cite{wang2019multi} & 77.3 & 85.3 & 90.5 & - & -  & 57.4 & 69.8 & 80.0 & - & - \\

MS{$^\dagger$} & 81.2 & 89.1 & 93.5 & 60.5 & 71.1  & 62.3 & 73.3 & 82.1 & 59.8 & 68.0 \\

\hline
EPS + Trip. + SH & 78.3 & 85.9 & 91.4 & 59.8 & 69.8 & 61.8  & 73.6 & 82.4 & 62.4 & 68.0 \\

EPS + Dist-Margin & \textbf{83.6} & \textbf{89.5} & \textbf{93.6} & 67.3 & \textbf{72.4} & \textbf{64.7} & \textbf{75.2} & \textbf{84.3} & 68.2 & \textbf{69.4} \\
EPS + MS & 82.9 & 89.4 & 93.2 & 60.0 & 72.0 & 63.3  & 74.2 & 82.5 & 61.2 & 68.2 \\
\end{tabular}

\vspace{+10pt}
\caption{Recall@k and NMI performance on Cars196 and CUB200-
2011. NMI+ indicate the NMI measurement when using 10 (number of classes) clusters. The EPS method improves in all cases. $^\dagger$: Our re-implemented version with the same embedding dimension.}
\label{table:1}
\end{table*}

\begin{table*}[!t]
\centering

\tablestyle{4.2pt}{1.05}
\begin{tabular}{l|x{22}x{22}x{22}x{22}|x{22}|x{22}x{22}x{22}x{22}|x{22}}
\multicolumn{1}{c|}{\multirow{2}{*}{model}} & \multicolumn{5}{c|}{Omniglot-letters} & \multicolumn{5}{c}{Omniglot-languages} \\
& R@1 & R@2 & R@4 & R@8 & NMI & R@1 & R@2 & R@4 & R@8 & NMI+ \\
\shline
Trip. + SH~\cite{Schroff2015}  & 49.4 & 60.0 & 69.2 & 76.9 & 66.2 & 71.0 & 80.2 & 87.6 & 92.4 & 38.7 \\
ProxyNCA~\cite{MovshovitzAttias2017} & 49.1 & 60.4 & 70.9 & 78.9 & 69.0 & 73.0 & 82.1 & 88.8 & 93.5 & 43.3 \\
Dist-Margin~\cite{Wu2017} & 49.4 & 61.1 & 70.1 & 79.2 & 68.9 & 73.2 & 82.3 & 89.1 & 94.0 & 43.5 \\
MS~\cite{wang2019multi} & 57.7 & 68.5 & 77.3 & 83.8 & 69.2 & 78.8 & 86.4 & 92.0 & 95.4 & 46.0 \\
\hline
EPS + Trip. + SH & \textbf{68.4} & \textbf{79.3} & \textbf{86.9} & \textbf{92.1} & \textbf{79.6} & \textbf{85.2} & \textbf{91.1} & \textbf{94.9} & \textbf{97.3} & \textbf{52.6}  \\
EPS + Dist-Margin & 66.2 & 76.7 & 84.8 & 90.3 & 77.9 & 83.0 & 89.4 & 93.6 & 96.4 & 50.7 \\
EPS + MS & 68.7 & 79.1 & 86.9 & 92.2 & 77.3 & 86.2 & 91.7 & 94.9 & 97.2 & 53.8
\end{tabular}

\vspace{+10pt}

\caption{Recall@k and NMI performance on Omniglot dataset. In both cases the training was done with only language labels. \textbf{Right:} evaluation on language labels. \textbf{Left:} evaluation on letter labels. NMI+ indicate the NMI measurement when using 30*(number of classes) clusters. The EPS method improves in both cases.}
\label{table:2}
\end{table*}

\subsubsection {Architecture details}
We use an embedding of size 128, and an input size of  $224 \times 224$ for the first two datasets, and $80 \times 80$ for the Omniglot dataset. For all the experiments we used the original bounding boxes without cropping around the object box. As a backbone for the embedding, we use ResNet50~\cite{He2016} with pretrained weights on imagenet. The backbone is followed by a global average pooling and a linear layer which reduces the dimension to the embedding size. Optimization is performed using Adam with a learning rate of $10^{-5}$, and the other parameters set to default values from


\subsubsection{Results}
We tested the EPS method with 3 different losses: Triplet~\cite{chechik2010large}, Margin~\cite{Wu2017} and Multi-Similarity (MS)~\cite{wang2019multi}. For the Margin loss experiment, we combine the EPS with distance sampling~\cite{Wu2017};  this could be done because the distance sampling only constrains on the negative samples, where our method only constrains on the positive samples. We set the margin $\alpha = 0.2$ and initialized $\beta = 1.2$  as in~\cite{Wu2017}. For the Triplet we combine EPS with semi-hard sampling~\cite{Schroff2015} by fixing the positive according to EPS and then using semi-hard sampling for choosing the negative examples. For the MS loss we replace the positive hard-mining method with EPS and use the same hard-negative method. We use the same hyper-paremeters as in~\cite{wang2019multi} $\alpha=2,\lambda=1,\beta=50$.

\begin{figure*}[t!]
    \begin{center}
        \includegraphics[width=1.\linewidth]{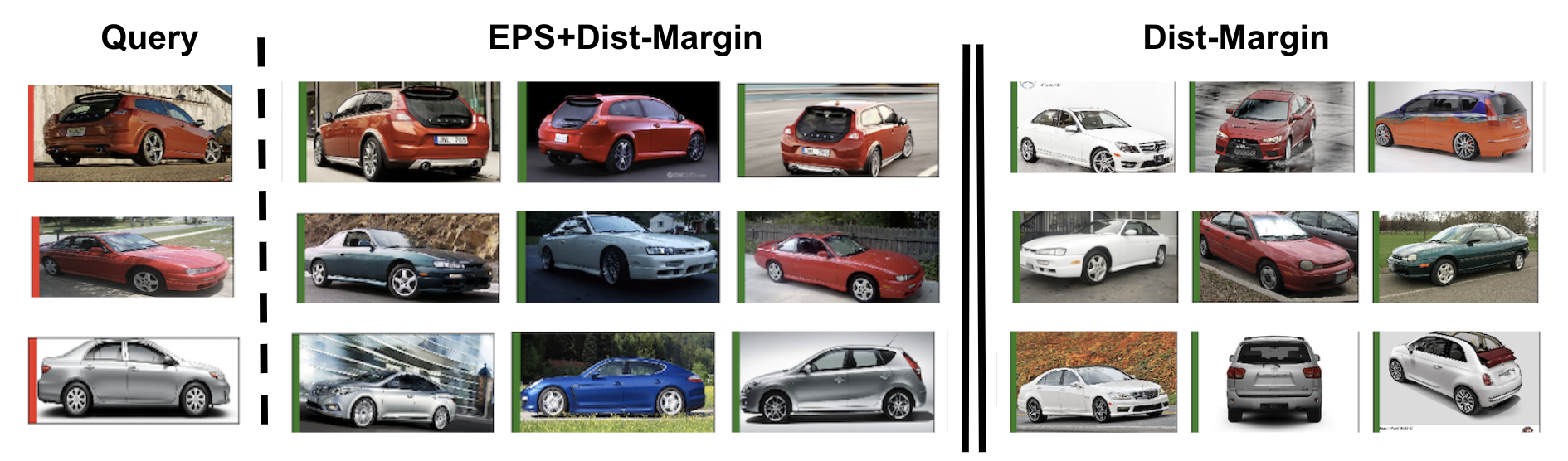}
    \caption{Retrieval results for randomly chosen query images in Cars196 dataset. Using EPS creates more homogeneous neighbourhood relationships with respect to the car viewpoint.} 
    \label{fig:carsEx}
    \end{center}
\end{figure*}

\begin{figure}[t!]
    \begin{center}
        \includegraphics[width=0.8\linewidth]{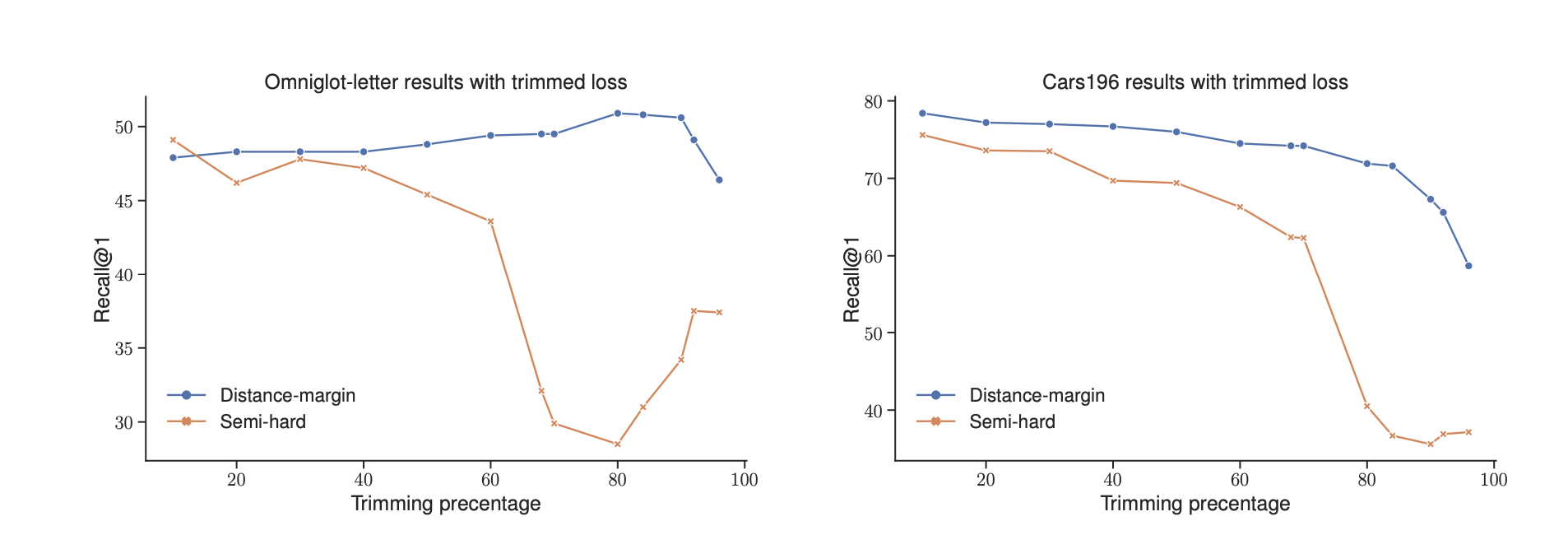}

    \caption{Recall@1 performance with Trimmed loss across varying trimming percentage. Except for small improvement in the Distance-margin case, there is no significant improvement when applying the Trimmed loss.} 
    \label{fig:trimmed}
    \end{center}
    \vspace{+10pt}
\end{figure}

Results are summarized in Tables \ref{table:1} and \ref{table:2}. We can see that EPS achieves the best performances. It is important to note that in the baseline models, when using Semi-hard sampling, the sampling strategy was done also on the positive part as suggest in the original papers. We see that replacing the semi-hard positive sampling with easy-positive sampling, improve results in all the experiments.  
The improvement gain becomes larger as the dataset classes can be partitioned more naturally to a small number of sub-clusters which are visually homogeneous. In Cars196 dataset it is the car viewpoint, where in Omniglot it is the letters in each language. As can be seen in Table \ref{table:2}, using EPS on the Omniglot dataset result in creating an embedding in which in most cases the nearest neighbor in the embedding consists of element of the same letter, although the network was trained without these labels. In Figure \ref{fig:carsEx} we can see a qualitatively comparison of CARS16 models results. EPS seems to create more homogeneous neighbourhood relationships with respect to the the viewpoint of the car. More results and comparisons can be find in Appendix C.


\begin{figure*}[t!]
    \begin{center}
        \includegraphics[width=0.9\linewidth]{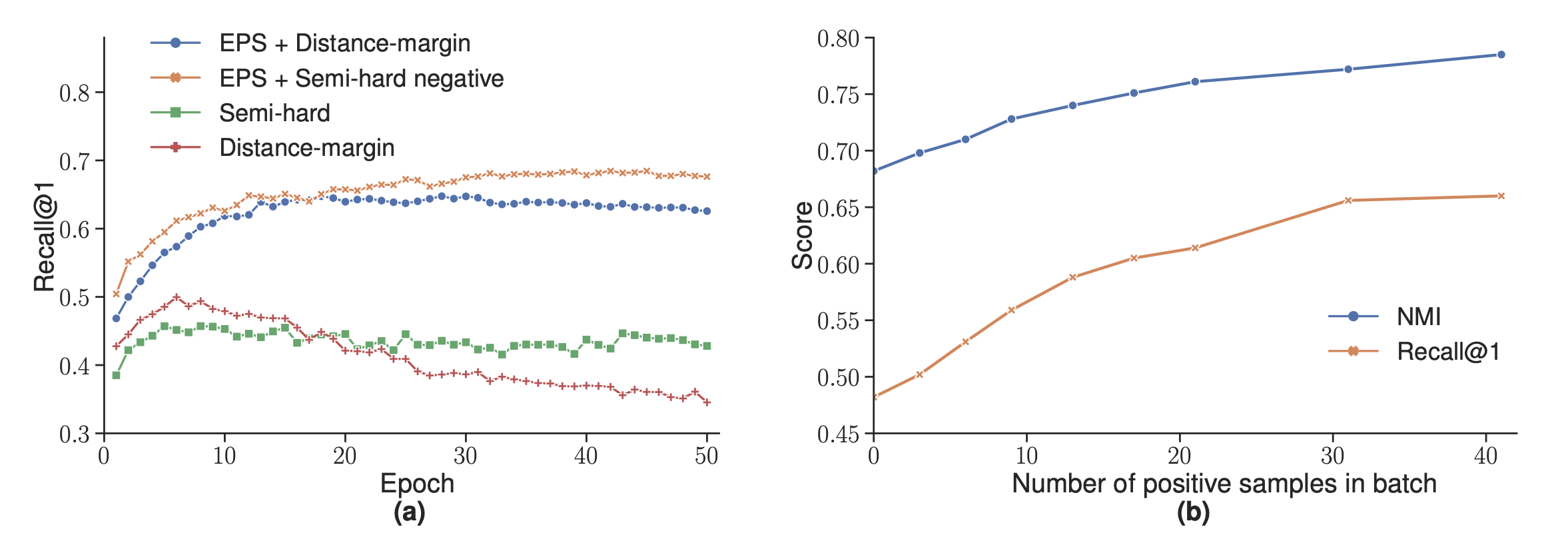}
    \caption{Results on Omniglot-letters. \textbf{(a)} Recall@1 performance of each model per epoch. \textbf{(b)} performance of  \textit{EPS + distance-margin} model on the Omniglot dataset, as a function of the number of positive samples in batch (where zero is equivalent to only using only distance sampling). Increasing the number the number of positive samples enhances the model performance.} 
    \label{fig:pos}
    \end{center}
\end{figure*}

\subsubsection{Positive batch size effect}
An important hyperparameter in our sampling method is the number of positive batch samples, from which we select the closest one in the embedding space to the anchor. If the class is visually diverse and the number of positive samples in batch is low, than with high probability the set of all the positive samples will not contain any visually similar image to the anchor. In case of the Omniglot experiment, the effect of this hyperparameter is clear; It determines the probability that the set of positive samples will include a sample from the same letter as the anchor letter. As can be seen in Figure \ref{fig:pos}(b), the performance of the model increases as the probability of having another sample with the same letter as the anchor increases.

\subsubsection{Trimmed Loss comparison}
The situation where a class consists of multiple modes can also be seen as a noisy data scenario with respect to the embedding loss, where positive tuples consisting of examples from different modes are considered as `bad` labelling. One approach to address noisy labels is by back-propagating the loss only on the k-elements in the batch with the lowest current loss~\cite{pmlr-v97-shen19e}. Although this approach resembles ~\cite{malisiewiczcvpr08}, the difference is that in~\cite{malisiewiczcvpr08} they apply the trimming only on the positive tuples. We test the effect of using Trimmed Loss on random sampled triplets with different level of trimming percentage. As can be seen in Figure \ref{fig:trimmed}, there is only a minor improvement when applying the loss on top of the distance-margin loss on the Omniglot-letters dataset. This emphasizes the importance of constraining the trimming to the positive sampling only.

\subsubsection{Embedding behavior on training sets}
\begin{table}[!t]
\centering

\tablestyle{3.75pt}{1.1}
\begin{tabular}{l|x{35}x{35}|x{35}x{35}|x{35}x{35}}
\multicolumn{1}{c|}{\multirow{2}{*}{}} & \multicolumn{2}{c|}{Language} & \multicolumn{2}{c|}{Letters} \\
& Semi-hard & Semi-hard+EPS & Semi-hard & Semi-hard+EPS  \\
\shline
NMI & \textbf{93.6} & 67.3 & 78.4 & \textbf{87.1}  \\
R@1 & \textbf{99.9} & 94.5 & 70.3 & \textbf{77.5}  \\
R@2 & \textbf{100} & 96.8 & 80.4 & \textbf{86.3}   \\
R@4 & \textbf{100} & 98.1 & 87.9 & \textbf{92.4}  \\
R@8 & \textbf{100} & 99.2   & 93.3 & \textbf{96.0}   \\
\end{tabular}

\caption{Results of semi-hard with/without EPS on the Omniglot training dataset. Without EPS the network feet almost perfectly to the training set. However, using EPS results in batter performances on the letters fine-grained task.   }
\label{omni-training}

\end{table}

\begin{figure*}
  \centering
  \begin{minipage}[b]{0.48\textwidth}
    \includegraphics[width=\textwidth]{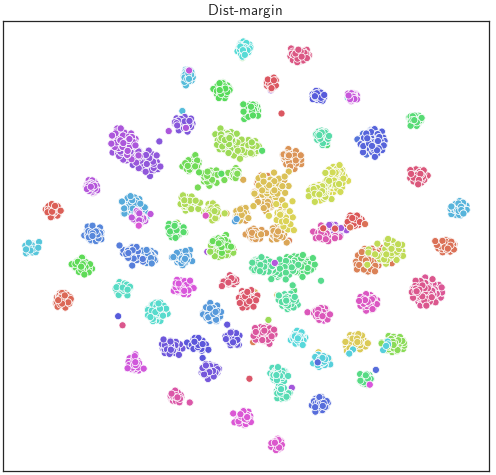}
  \end{minipage}
  \hfill
  \begin{minipage}[b]{0.48\textwidth}
    \includegraphics[width=\textwidth]{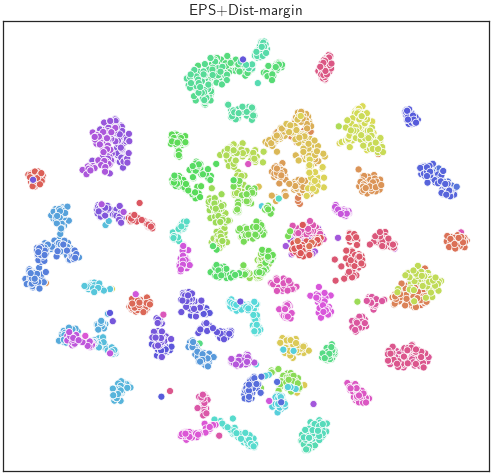}
  \end{minipage}
  
      \caption{t-SNE visualization of Cars196 training classes (each class has a different color). Training with EPS results in more diverse classes appearance.     }
      \label{fig:tsne}
\end{figure*}

The class-collapsing phenomena also occur in the training process of the image retrieval datasets. Figure \ref{fig:tsne} visualise the t-SNE embedding~\cite{maaten2008visualizing} of Cars196 training classes. As can be seen, when training without EPS each class fits well to a bivariate normal-distribution with small variance and different means. Training with EPS result in more diverse distributions and in some of the classes fits batter to a mixture of multiple different distributions. 

This can also be measured qualitatively on the Omniglod detest; although training without the EPS results in batter overfitting to training samples, the results on the letters fine-grained task are significantly inferior comparing to training with the EPS (Table \ref{omni-training}). It is also important to note the low NMI score when using EPS with the number of clusters equal to the number of languages, and the increment of this score when increasing the number of clusters to the number of letters. This indicates that training with EPS results in more homogeneous small clusters, which are more blended in the embedding space comparing to training without EPS.  

\section{Conclusion}
In this work we we investigate the class collapsing phenomena with respect to popular embedding losses such as the Triplet loss and the Margin loss. While in clean environments there is a diverse and rich family of optimal solutions, when noise is present, the optimal solution collapses to a degenerate embedding.
We propose a simple solution to this issue based on 'easy' positive sampling, and prove that indeed adding this sampling results in non-degenerate embeddings. We also compare and evaluate our method on standard image retrieval datasets and demonstrate a consistent performance boost on all of them. While our analysis and results have been limited to metric learning frameworks, we believe that this type of noisy analysis might be useful in other settings, and can better reflect the training dynamic of neural-networks on real-world datasets.

\textbf{Acknowledgements.} Prof. Darrell’s group was supported by BAIR.

\printbibliography

@String{CVPR = "Proc. Conf. Comput. Vision Pattern Recognition"}

@String{ECCV = "European Conf. Comput. Vision"}

@String{ICCV = "Proc. Int. Conf. Comput. Vision"}

@String{WACV = "Winter Conf. on App. of Comput. Vision"}

@Article{Lowe1995,
  author    = {David G. Lowe},
  journal   = {Neural Computation},
  title     = {Similarity Metric Learning for a Variable-Kernel Classifier},
  year      = {1995},
  month     = {jan},
  number    = {1},
  pages     = {72--85},
  volume    = {7},
  doi       = {10.1162/neco.1995.7.1.72},
  publisher = {{MIT} Press - Journals},
}

@Article{Qian2019,
  author      = {Qian, Qi  and Shang, Lei  and Sun, Baigui  and Hu, Juhua  and Li, Hao  and Jin, Rong },
  title       = {SoftTriple Loss: Deep Metric Learning Without Triplet Sampling},
  date        = {2019-09-11},
  eprint      = {1909.05235v2},
  eprintclass = {cs.CV},
  eprinttype  = {arXiv},
}

@InProceedings{pmlr-v97-shen19e,
  author    = {Shen, Yanyao and Sanghavi, Sujay},
  booktitle = {Proceedings of the 36th International Conference on Machine Learning},
  title     = {Learning with Bad Training Data via Iterative Trimmed Loss Minimization},
  year      = {2019},
  address   = {Long Beach, California, USA},
  editor    = {Chaudhuri, Kamalika and Salakhutdinov, Ruslan},
  month     = {09--15 Jun},
  pages     = {5739--5748},
  publisher = {PMLR},
  series    = {Proceedings of Machine Learning Research},
  volume    = {97},
  url       = {http://proceedings.mlr.press/v97/shen19e.html},
}

@InProceedings{malisiewiczcvpr08,
  author    = {Tomasz Malisiewicz and Alexei A. Efros},
  booktitle = {CVPR},
  title     = {Recognition by Association via Learning Per-exemplar Distances},
  year      = {2008},
  month     = {June},
}

@inproceedings{Hadsell,
  doi = {10.1109/cvpr.2006.100},
  url = {https://doi.org/10.1109/cvpr.2006.100},
  publisher = {{IEEE}},
  author = {R. Hadsell and S. Chopra and Y. LeCun},
  title = {Dimensionality Reduction by Learning an Invariant Mapping},
  booktitle = {2006 {IEEE} Computer Society Conference on Computer Vision and Pattern Recognition - Volume 2 (CVPR06)}
}

@InProceedings{KrauseStarkDengFei-Fei_3DRR2013,
  author    = {Jonathan Krause and Michael Stark and Jia Deng and Li Fei-Fei},
  booktitle = {4th International IEEE Workshop on 3D Representation and Recognition (3dRR-13)},
  title     = {3D Object Representations for Fine-Grained Categorization},
  year      = {2013},
  address   = {Sydney, Australia},
}

@TechReport{WahCUB_200_2011,
  author      = {Wah, C. and Branson, S. and Welinder, P. and Perona, P. and Belongie, S.},
  institution = {California Institute of Technology},
  title       = {{The Caltech-UCSD Birds-200-2011 Dataset}},
  year        = {2011},
  number      = {CNS-TR-2011-001},
}

@Article{Lake2015,
  author    = {B. M. Lake and R. Salakhutdinov and J. B. Tenenbaum},
  journal   = {Science},
  title     = {Human-level concept learning through probabilistic program induction},
  year      = {2015},
  month     = {dec},
  number    = {6266},
  pages     = {1332--1338},
  volume    = {350},
  doi       = {10.1126/science.aab3050},
  publisher = {American Association for the Advancement of Science ({AAAS})},
}

@InCollection{NIPS2016_6200,
  author    = {Sohn, Kihyuk},
  booktitle = {Advances in Neural Information Processing Systems 29},
  publisher = {Curran Associates, Inc.},
  title     = {Improved Deep Metric Learning with Multi-class N-pair Loss Objective},
  year      = {2016},
  editor    = {D. D. Lee and M. Sugiyama and U. V. Luxburg and I. Guyon and R. Garnett},
  pages     = {1857--1865},
  url       = {http://papers.nips.cc/paper/6200-improved-deep-metric-learning-with-multi-class-n-pair-loss-objective.pdf},
}

@InProceedings{Wu2017,
  author    = {Chao-Yuan Wu and R. Manmatha and Alexander J. Smola and Philipp Krahenbuhl},
  booktitle = {2017 {IEEE} International Conference on Computer Vision ({ICCV})},
  title     = {Sampling Matters in Deep Embedding Learning},
  year      = {2017},
  month     = {oct},
  publisher = {{IEEE}},
  doi       = {10.1109/iccv.2017.309},
}

@conference {6321,
	title = {MIC: Mining Interclass Characteristics for Improved Metric Learning},
	booktitle = {Proceedings of the Intl. Conf. on Computer Vision (ICCV)},
	year = {2019},
	author = {Biagio Brattoli and Karsten Roth and Bj{\"o}rn Ommer}
}

@InProceedings{Schroff2015,
  author    = {Florian Schroff and Dmitry Kalenichenko and James Philbin},
  booktitle = {2015 {IEEE} Conference on Computer Vision and Pattern Recognition ({CVPR})},
  title     = {{FaceNet}: A unified embedding for face recognition and clustering},
  year      = {2015},
  month     = {jun},
  publisher = {{IEEE}},
  doi       = {10.1109/cvpr.2015.7298682},
}

@Article{Jegou2011,
  author    = {H J{\'{e}}gou and M Douze and C Schmid},
  journal   = {{IEEE} Transactions on Pattern Analysis and Machine Intelligence},
  title     = {Product Quantization for Nearest Neighbor Search},
  year      = {2011},
  month     = {jan},
  number    = {1},
  pages     = {117--128},
  volume    = {33},
  doi       = {10.1109/tpami.2010.57},
  publisher = {Institute of Electrical and Electronics Engineers ({IEEE})},
}

@Book{Manning2008,
  author    = {Christopher D. Manning and Prabhakar Raghavan and Hinrich Schutze},
  publisher = {Cambridge University Press},
  title     = {Introduction to Information Retrieval},
  year      = {2008},
  doi       = {10.1017/cbo9780511809071},
}

@Article{Lecun1998,
  author    = {Y. Lecun and L. Bottou and Y. Bengio and P. Haffner},
  journal   = {Proceedings of the {IEEE}},
  title     = {Gradient-based learning applied to document recognition},
  year      = {1998},
  number    = {11},
  pages     = {2278--2324},
  volume    = {86},
  doi       = {10.1109/5.726791},
  publisher = {Institute of Electrical and Electronics Engineers ({IEEE})},
}

@InProceedings{MovshovitzAttias2017,
  author    = {Yair Movshovitz-Attias and Alexander Toshev and Thomas K. Leung and Sergey Ioffe and Saurabh Singh},
  booktitle = {2017 {IEEE} International Conference on Computer Vision ({ICCV})},
  title     = {No Fuss Distance Metric Learning Using Proxies},
  year      = {2017},
  month     = {oct},
  publisher = {{IEEE}},
  doi       = {10.1109/iccv.2017.47},
}

@Article{Roth2020,
  author      = {Karsten Roth and Timo Milbich and Samarth Sinha and Prateek Gupta and Björn Ommer and Joseph Paul Cohen},
  title       = {Revisiting Training Strategies and Generalization Performance in Deep Metric Learning},
  date        = {2020-02-19},
  eprint      = {2002.08473v7},
  eprintclass = {cs.CV},
  eprinttype  = {arXiv},
}

@InCollection{He2016,
  author    = {Kaiming He and Xiangyu Zhang and Shaoqing Ren and Jian Sun},
  booktitle = {Computer Vision {\textendash} {ECCV} 2016},
  publisher = {Springer International Publishing},
  title     = {Identity Mappings in Deep Residual Networks},
  year      = {2016},
  pages     = {630--645},
  doi       = {10.1007/978-3-319-46493-0_38},
}

@inproceedings{chopra2005learning,
  title={Learning a similarity metric discriminatively, with application to face verification},
  author={Chopra, Sumit and Hadsell, Raia and LeCun, Yann},
  booktitle={2005 IEEE Computer Society Conference on Computer Vision and Pattern Recognition (CVPR'05)},
  volume={1},
  pages={539--546},
  year={2005},
  organization={IEEE}
}

@inproceedings{taigman2014deepface,
  title={Deepface: Closing the gap to human-level performance in face verification},
  author={Taigman, Yaniv and Yang, Ming and Ranzato, Marc'Aurelio and Wolf, Lior},
  booktitle={Proceedings of the IEEE conference on computer vision and pattern recognition},
  pages={1701--1708},
  year={2014}
}

@incollection{NIPS2014_5416,
title = {Deep Learning Face Representation by Joint Identification-Verification},
author = {Sun, Yi and Chen, Yuheng and Wang, Xiaogang and Tang, Xiaoou},
booktitle = {Advances in Neural Information Processing Systems 27},
editor = {Z. Ghahramani and M. Welling and C. Cortes and N. D. Lawrence and K. Q. Weinberger},
pages = {1988--1996},
year = {2014},
publisher = {Curran Associates, Inc.}
}

@inproceedings{hoffer2015deep,
  title={Deep metric learning using triplet network},
  author={Hoffer, Elad and Ailon, Nir},
  booktitle={International Workshop on Similarity-Based Pattern Recognition},
  pages={84--92},
  year={2015},
  organization={Springer}
}

@inproceedings{wang2014learning,
  title={Learning fine-grained image similarity with deep ranking},
  author={Wang, Jiang and Song, Yang and Leung, Thomas and Rosenberg, Chuck and Wang, Jingbin and Philbin, James and Chen, Bo and Wu, Ying},
  booktitle={Proceedings of the IEEE Conference on Computer Vision and Pattern Recognition},
  pages={1386--1393},
  year={2014}
}

@inproceedings{gordo2016deep,
  title={Deep image retrieval: Learning global representations for image search},
  author={Gordo, Albert and Almaz{\'a}n, Jon and Revaud, Jerome and Larlus, Diane},
  booktitle={European conference on computer vision},
  pages={241--257},
  year={2016},
  organization={Springer}
}

@inproceedings{frome2013devise,
  title={Devise: A deep visual-semantic embedding model},
  author={Frome, Andrea and Corrado, Greg S and Shlens, Jon and Bengio, Samy and Dean, Jeff and Ranzato, Marc'Aurelio and Mikolov, Tomas},
  booktitle={Advances in neural information processing systems},
  pages={2121--2129},
  year={2013}
}

@inproceedings{Wang_UnsupICCV2015,
    Author = {Xiaolong Wang and Abhinav Gupta},
    Title = {Unsupervised Learning of Visual Representations using Videos},
    Booktitle = {ICCV},
    Year = {2015},
}

@inproceedings{he2019momentum,
  title={Momentum contrast for unsupervised visual representation learning},
  author={He, Kaiming and Fan, Haoqi and Wu, Yuxin and Xie, Saining and Girshick, Ross},
  booktitle={Proceedings of the IEEE Conference on Computer Vision and Pattern Recognition},
  year={2020}
}

@inproceedings{chen2020simple,
  title={A simple framework for contrastive learning of visual representations},
  author={Chen, Ting and Kornblith, Simon and Norouzi, Mohammad and Hinton, Geoffrey},
  journal={arXiv preprint arXiv:2002.05709},
  year={2020}
}

@inproceedings{wu2018unsupervised,
  title={Unsupervised feature learning via non-parametric instance discrimination},
  author={Wu, Zhirong and Xiong, Yuanjun and Yu, Stella X and Lin, Dahua},
  booktitle={Proceedings of the IEEE Conference on Computer Vision and Pattern Recognition},
  pages={3733--3742},
  year={2018}
}

@inproceedings{shrivastava2016training,
  title={Training region-based object detectors with online hard example mining},
  author={Shrivastava, Abhinav and Gupta, Abhinav and Girshick, Ross},
  booktitle={Proceedings of the IEEE conference on computer vision and pattern recognition},
  pages={761--769},
  year={2016}
}

@inproceedings{cubuk2019autoaugment,
  title={Autoaugment: Learning augmentation strategies from data},
  author={Cubuk, Ekin D and Zoph, Barret and Mane, Dandelion and Vasudevan, Vijay and Le, Quoc V},
  booktitle={Proceedings of the IEEE conference on computer vision and pattern recognition},
  pages={113--123},
  year={2019}
}

@inproceedings{singh2017hide,
  title={Hide-and-seek: Forcing a network to be meticulous for weakly-supervised object and action localization},
  author={Singh, Krishna Kumar and Lee, Yong Jae},
  booktitle={2017 IEEE international conference on computer vision (ICCV)},
  pages={3544--3553},
  year={2017},
  organization={IEEE}
}

@inproceedings{Wang_afrcnnCVPR2017,
    Author = {Xiaolong Wang and Abhinav Shrivastava and Abhinav Gupta},
    Title = {A-Fast-RCNN: Hard Positive Generation via Adversary for Object Detection},
    Booktitle = {CVPR},
    Year = {2017},
}

@article{loshchilov2015online,
  title={Online batch selection for faster training of neural networks},
  author={Loshchilov, Ilya and Hutter, Frank},
  journal={arXiv preprint arXiv:1511.06343},
  year={2015}
}

@article{zhang2016understanding,
  title={Understanding deep learning requires rethinking generalization},
  author={Zhang, Chiyuan and Bengio, Samy and Hardt, Moritz and Recht, Benjamin and Vinyals, Oriol},
  journal={arXiv preprint arXiv:1611.03530},
  year={2016}
}

@article{recht2018cifar,
  title={Do cifar-10 classifiers generalize to cifar-10?},
  author={Recht, Benjamin and Roelofs, Rebecca and Schmidt, Ludwig and Shankar, Vaishaal},
  journal={arXiv preprint arXiv:1806.00451},
  year={2018}
}

@article{recht2019imagenet,
  title={Do imagenet classifiers generalize to imagenet?},
  author={Recht, Benjamin and Roelofs, Rebecca and Schmidt, Ludwig and Shankar, Vaishaal},
  journal={arXiv preprint arXiv:1902.10811},
  year={2019}
}

@article{reed2014training,
  title={Training deep neural networks on noisy labels with bootstrapping},
  author={Reed, Scott and Lee, Honglak and Anguelov, Dragomir and Szegedy, Christian and Erhan, Dumitru and Rabinovich, Andrew},
  journal={arXiv preprint arXiv:1412.6596},
  year={2014}
}

@inproceedings{simo2015discriminative,
  title={Discriminative learning of deep convolutional feature point descriptors},
  author={Simo-Serra, Edgar and Trulls, Eduard and Ferraz, Luis and Kokkinos, Iasonas and Fua, Pascal and Moreno-Noguer, Francesc},
  booktitle={Proceedings of the IEEE International Conference on Computer Vision},
  pages={118--126},
  year={2015}
}

@inproceedings{oh2016deep,
  title={Deep metric learning via lifted structured feature embedding},
  author={Oh Song, Hyun and Xiang, Yu and Jegelka, Stefanie and Savarese, Silvio},
  booktitle={Proceedings of the IEEE conference on computer vision and pattern recognition},
  pages={4004--4012},
  year={2016}
}

@article{parkhi2015deep,
  title={Deep face recognition},
  author={Parkhi, Omkar M and Vedaldi, Andrea and Zisserman, Andrew},
  year={2015},
  publisher={British Machine Vision Association}
}

@inproceedings{sun2017revisiting,
  title={Revisiting unreasonable effectiveness of data in deep learning era},
  author={Sun, Chen and Shrivastava, Abhinav and Singh, Saurabh and Gupta, Abhinav},
  booktitle={Proceedings of the IEEE international conference on computer vision},
  pages={843--852},
  year={2017}
}

@inproceedings{malach2017decoupling,
  title={Decoupling" when to update" from" how to update"},
  author={Malach, Eran and Shalev-Shwartz, Shai},
  booktitle={Advances in Neural Information Processing Systems},
  pages={960--970},
  year={2017}
}

@article{jiang2017mentornet,
  title={Mentornet: Learning data-driven curriculum for very deep neural networks on corrupted labels},
  author={Jiang, Lu and Zhou, Zhengyuan and Leung, Thomas and Li, Li-Jia and Fei-Fei, Li},
  journal={arXiv preprint arXiv:1712.05055},
  year={2017}
}

@inproceedings{scott2013classification,
  title={Classification with asymmetric label noise: Consistency and maximal denoising},
  author={Scott, Clayton and Blanchard, Gilles and Handy, Gregory},
  booktitle={Conference On Learning Theory},
  pages={489--511},
  year={2013}
}

@inproceedings{natarajan2013learning,
  title={Learning with noisy labels},
  author={Natarajan, Nagarajan and Dhillon, Inderjit S and Ravikumar, Pradeep K and Tewari, Ambuj},
  booktitle={Advances in neural information processing systems},
  pages={1196--1204},
  year={2013}
}

@article{khetan2017learning,
  title={Learning from noisy singly-labeled data},
  author={Khetan, Ashish and Lipton, Zachary C and Anandkumar, Anima},
  journal={arXiv preprint arXiv:1712.04577},
  year={2017}
}

@article{chechik2010large,
  author = {Chechik, Gal and Sharma, Varun and Shalit, Uri and Bengio, Samy},
  doi = {10.1145/1756006.1756042},
  journal = {Journal of Machine Learning Research},
  pages = {1109-1135},
  title = {Large Scale Online Learning of Image Similarity Through Ranking},
  volume = 11,
  year = 2010
}

@inproceedings{wang2019multi,
title={Multi-Similarity Loss with General Pair Weighting for Deep Metric Learning},
author={Wang, Xun and Han, Xintong and Huang, Weilin and Dong, Dengke and Scott, Matthew R},
booktitle={Proceedings of the IEEE Conference on Computer Vision and Pattern Recognition},
pages={5022--5030},
year={2019}
}

@misc{musgrave2020metric,
    title={A Metric Learning Reality Check},
    author={Kevin Musgrave and Serge Belongie and Ser-Nam Lim},
    year={2020},
    eprint={2003.08505},
    archivePrefix={arXiv},
    primaryClass={cs.CV}
}

@misc{1911.12528,
Author = {Istvan Fehervari and Avinash Ravichandran and Srikar Appalaraju},
Title = {Unbiased Evaluation of Deep Metric Learning Algorithms},
Year = {2019},
Eprint = {arXiv:1911.12528},
}

@inproceedings{arandjelovic2016netvlad,
  title={NetVLAD: CNN architecture for weakly supervised place recognition},
  author={Arandjelovic, Relja and Gronat, Petr and Torii, Akihiko and Pajdla, Tomas and Sivic, Josef},
  booktitle={Proceedings of the IEEE conference on computer vision and pattern recognition},
  pages={5297--5307},
  year={2016}
}

@article{maaten2008visualizing,
  author = {van der Maaten, L.J.P. and Hinton, G.E.},
  title = {Visualizing High-Dimensional Data Using t-SNE},
  year = 2008
}

@inproceedings{DBLP:conf/wacv/XuanSP20,
  author    = {Hong Xuan and
               Abby Stylianou and
               Robert Pless},
  title     = {Improved Embeddings with Easy Positive Triplet Mining},
  booktitle = {{IEEE} Winter Conference on Applications of Computer Vision, {WACV}
               2020, Snowmass Village, CO, USA, March 1-5, 2020},
  pages     = {2463--2471},
  publisher = {{IEEE}},
  year      = {2020},
  url       = {https://doi.org/10.1109/WACV45572.2020.9093432},
  doi       = {10.1109/WACV45572.2020.9093432},
  bibsource = {dblp computer science bibliography, https://dblp.org}
}

@article{2cfc44fb6aa842f5ac03e51409667171,
title = "Aleatoric or epistemic? Does it matter?",
author = "{Der Kiureghian}, Armen and Ditlevsen, {Ove Dalager}",
year = "2009",
doi = "10.1016/j.strusafe.2008.06.020",
language = "English",
volume = "31",
pages = "105--112",
journal = "Structural Safety",
issn = "0167-4730",
publisher = "Elsevier",
number = "2",
}

@article{citeulike:3561150,
  author = {Cybenko, G.},
  day = 1,
  doi = {10.1007/BF02551274},
  issn = {0932-4194},
  journal = {Mathematics of Control, Signals, and Systems (MCSS)},
  month = dec,
  number = 4,
  pages = {303--314},
  publisher = {Springer London},
  title = {{Approximation by superpositions of a sigmoidal function}},
  url = {http://dx.doi.org/10.1007/BF02551274},
  volume = 2,
  year = 1989
}

\section*{Appendix}
\subsection*{A:\quad  Proofs for the Theorems in subsection 4.2 }
We use all notions defined in subsection 4.2

\begin{theorem}[1] 
\label{Lagrange}
Let $f:X \xrightarrow{}\mathbb{R}^m$ be an embedding, which minimizes $ \mathbb{E}\mathbb{O}_{trip}(f)$, then $f$ has the class-collapsing property with respect to all classes.
\end{theorem}

\begin{proof}

Define a new random variables such that for every $1 \leq r_1,r_2 \leq t$: \[h_{r_1,r_2}(Y,Z)=\begin{cases}
1 & Y=r_1\wedge Z=r_2\\
0 & else
\end{cases}\] observe that 
\[ \bar{\delta}_{Y_{1},Y_{2}}\cdot(1-\bar{\delta}_{Y_{1},Y_{3}})=\sum_{\underset{r1\neq r_{2}}{1\leq r_{1},r_{2}\leq t}}\mathbf{1}_{Y_{1}=r_{1}}\cdot h_{r_{1},r_{2}}(Y_{2},Y_{3})=\sum_{\underset{r1\neq r_{2}}{1\leq r_{1},r_{2}\leq t}}\mathbf{1}_{Y_{1}=r_{2}}\cdot h_{r_{1},r_{2}}(Y_{3},Y_{2})
 \]
Since the variables are independent \[\mathbb{E}(\bar{\delta}_{Y_{1},Y_{2}}\cdot(1-\bar{\delta}_{Y_{1},Y_{3}}))=\frac{1}{2}\cdot\sum_{\underset{r1\neq r_{2}}{1\leq r_{1},r_{2}\leq t}}\mathbb{E}(\mathbf{1}_{Y_{1}=r_{1}})\cdot\mathbb{E}(h_{r_{1},r_{2}}(Y_{2},Y_{3}))+\mathbb{E}(\mathbf{1}_{Y_{1}=r_{2}})\cdot\mathbb{E}(h_{r_{1},r_{2}}(Y_{3},Y_{2})).\] 
Define: $\bar{D}(x_{1},x_{2},x_{3}):=(D_{x_{1},x_{2}}-D_{x_{1},x_{3}}+\alpha)_{+}$

Rearranging the terms we get 

\begin{align*} & n^{3}\cdot\mathbb{E}\mathbb{O}_{trip}(f)=\sum_{x_{1},x_{2},x_{3}\in X}(\mathbb{E}(\bar{\delta}_{Y_{1},Y_{2}}\cdot(1-\bar{\delta}_{Y_{1},Y_{3}}))\cdot\bar{D}(x_{1},x_{2},x_{3})=\\
 & \sum_{\underset{1\le r_{1}\neq r_{2}\leq t}{x_{1},x_{2},x_{3}\in X}}\left(\mathbb{E}(\mathbf{1}_{Y_{1}=r_{1}})\cdot\mathbb{E}(h_{r_{1},r_{2}}((Y_{2},Y_{3}))+\mathbb{E}(\mathbf{1}_{Y_{1}=r_{2}})\cdot\mathbb{E}(h_{r_{1},r_{2}}(Y_{3},Y_{2}))\right)\cdot\bar{D}(x_{1},x_{2},x_{3})=\\
 & \sum_{\underset{1\le r_{1}\neq r_{2}\leq t}{x_{1},x_{2},x_{3}\in X}}\mathbb{E}(h_{r_{1,},r_{2}}(Y_{2},Y_{3}))\cdot\left(\mathbb{E}(\mathbf{1}_{Y_{1}=r_{1}})\cdot\bar{D}(x_{1},x_{2},x_{3})+\mathbb{E}(\mathbf{1}_{Y_{1}=r_{2}})\cdot\bar{D}(x_{1},x_{3},x_{2})\right)=
\end{align*}

Therefore, if \[K(i,j,k,r_{1},r_{2})=\left(\mathbb{E}(\mathbf{1}_{Y_{1}=r_{1}})\cdot\bar{D}(x_{i},x_{j},x_{k})+\mathbb{E}(\mathbf{1}_{Y_{1}=r_{2}})\cdot\bar{D}(x_{1},x_{k},x_{j})\right),\]
then $ \mathbb{E}\mathbb{O}_{trip}(f)$ can be written as
\[\mathbb{E}\mathbb{O}_{trip}(f)=\frac{1}{n^{3}}\sum_{\underset{1\le r_{1}\neq r_{2}\leq t}{1\leq i,j,k\leq n}}\mathbb{E}(h_{r_{1},r_{2}}(Y_{j},Y_{k}))\cdot K(i,j,k,r_{1},r_{2})\]

For every $x_i\in X$, define:  \[(\mathbb{E}\mathbb{O}_{trip}(f))_{x_{i}}=\frac{1}{n^{2}}\cdot\sum_{\underset{1\le r_{1}\neq r_{2}\leq t}{1\leq j,k\leq n}}(\mathbb{E}(h_{r_{1},r_{2}}(Y_{j},Y_{k}))\cdot K(x_{i},x_{j},x_{k},r_{1},r_{2})\]
Let $f : X \xrightarrow{}\mathbb{R}^m$ be an embedding, fix $1\leq r\leq t$ and $x_i\in A_r,\quad
x_j,x_k\in X$ with \[ \Arrowvert\ f(x_i)- f(x_j) \Arrowvert\ = w,\quad \Arrowvert\ f(x_i)- f(x_k) \Arrowvert\ = h .\]
By definition: \[K(i,j,k,r_{1},r_{2})=\begin{cases}
p\cdot(h-w+\alpha)_{+}+q(w-h+\alpha)_{+} & r_{1}=r\wedge r_{2}\neq r\\
q\cdot(h-w+\alpha)_{+}+p(w-h+\alpha)_{+} & r_{2}=r\wedge r_{1}\neq r\\
p\cdot(h-w+\alpha)_{+}+p(w-h+\alpha)_{+} & r_{1}=r\wedge r_{2}=r\\
q\cdot(h-w+\alpha)_{+}+q(w-h+\alpha)_{+} & r_{1}\neq r\wedge r_{2}\neq r
\end{cases}\]
Since $0.5<p<1$, in order to get minimal $K(i,j,k,r_1,r_2)$ value, $h$ and $w$ must satisfy $|h-w|\leq\alpha$. In this case we have \[K(i,j,k,r_1,r_2)=\begin{cases}
(p+q)\cdot\alpha+(h-w)(p-q) & r_{1}=r\wedge r_{2}\neq r\\
(p+q)\cdot\alpha+(w-h)(p-q) & r_{2}=r\wedge r_{1}\neq r\\
2\cdot\alpha & r_{1}=r\wedge r_{2}=r\\
2\cdot\alpha & r_{1}\neq r\wedge r_{2}\neq r
\end{cases}\]
Therefore,
\begin{align*}
 & \sum_{r_{2}\in\{1,.r-1,r+1,.t\}}(\mathbb{E}(h_{r,r_{2}}(Y_{j},Y_{k}))\cdot K(x_{i},x_{j},x_{k},r_{1},r_{2})+(\mathbb{E}(h_{r_{2},r}(Y_{j},Y_{k}))\cdot K(x_{i},x_{j},x_{k},r_{1},r_{2})=\\
 & =(p+q)\cdot\alpha(\sum_{r_{2}\in\{1,.r-1,r+1,.t\}}(\mathbb{E}(h_{r,r_{2}}(Y_{j},Y_{k})+(\mathbb{E}(h_{r_{2},r}(Y_{j},Y_{k})))+\\
 & (h-w)(p-q))(\sum_{r_{2}\in\{1,.r-1,r+1,.t\}}\mathbb{E}(h_{r,r_{2}}(Y_{j},Y_{k}))-\mathbb{E}(h_{r_{2},r}(Y_{j},Y_{k}))
\end{align*}

\par
We split to three cases:
\begin{enumerate}
\item If $x_j,x_k\in A_r$ or $x_j,x_k\notin A_r$ then: $\mathbb{E}(h_{r,r_{2}}(Y_{j},Y_{k}))=\mathbb{E}(h_{r_{2},r}(Y_{j},Y_{k}))$.
Hence, \[(h-w)(p-q))(\sum_{r_{2}\in\{1,.r-1,r+1,.t\}}\mathbb{E}(h_{r,r_{2}}(Y_{j},Y_{k}))-\mathbb{E}(h_{r_{2},r}(Y_{j},Y_{k}))=0\]
\item If $x_j\in A_r$ and $x_k\notin A_r$, then $\mathbb{E}(h_{r,r_{2}}(Y_{j},Y_{k}))>\mathbb{E}(h_{r_{2},r}(Y_{j},Y_{k}))$,
therefore \[(h-w)(p-q))(\sum_{r_{2}\in\{1,.r-1,r+1,.t\}}\mathbb{E}(h_{r,r_{2}}(Y_{j},Y_{k}))-\mathbb{E}(h_{r_{2},r}(Y_{j},Y_{k}))\]

 Since $p>0.5$ and $|h-w|\leq\alpha$,the minimal value is achieved whenever $h=0$
and $w=\alpha$.
 \item In the same way if  $x_k\in A_r$ and $x_j\notin A_r$, then
 $\mathbb{E}(h_{h_{r_{2},r}}(Y_{j},Y_{k}))=\mathbb{E}(h_{r,r_{2}}(Y_{j},Y_{k}))$ and the minimal value is achieved whenever $h=\alpha$ and $w=0$.
\end{enumerate}
In conclusion, if $x_i\in A_r$, an embedding $f^{*}$ satisfies \[(\mathbb{E}\mathbb{O}_{trip}(f^{*}))_{x_i}= \min\{(\mathbb{E}\mathbb{O}_{trip}(f))_{x_i} | f : X \xrightarrow{}\mathbb{R}^m\}\]
iff $f^*(x_j)=f^*(x_i)$ for every $x_j\in A_r$, and $ \Arrowvert\ f^*(x_j)-f^*(x_i) \Arrowvert\ =\alpha$
for every $x_j\notin A_r$.
\end{proof}

We will now prove the same theorem with respect to the margin loss.

\begin{theorem}[2]
\label{Lagrange}
Let $f: X \xrightarrow{}\mathbb{R}^m$ be an embedding, which minimizes \[\mathbb{E}\mathbb{O}_{margin}(f,\beta)=\frac{1}{n^{2}}\sum_{x_{i},x_{j}\in X}\mathbb{E}\mathcal{L}_{margin}^{f}(x_{i},x_{j}),\] then $f$ has the class-collapsing property with respect to all classes.
\end{theorem}

\begin{proof}

Observe that if $x_{i},x_{j}\in A_r$, then 
 \[\mathbb{E}\mathcal{L}_{margin}^{f}(x_{i},x_{j})=p\cdot(D_{x_{i},x_{j}}-\beta_{x_{i}}+\alpha)_{+}+(1-p)\cdot(\beta_{x_{i}}-D_{x_{i},x_{j}}+\alpha)_{+}\]
Since $0<p<1$, then the maximal value is achieved whenever $|D_{x_{i},x_{j}}-\beta_{x_{i}}|\leq\alpha$,
in this case:
\[\mathbb{E}\mathcal{L}_{margin}^{f}(x_{i},x_{j})=(2p-1)\cdot(D_{x_{i},x_{j}}-\beta_{x_{i}}).\]
In the same way in case $x_{i}\in A_r$ and $x_{j}\notin A_r$ then:
\[\mathbb{E}\mathcal{L}_{margin}^{f}(x_{i},x_{j})=(2p-1)\cdot(\beta_{x_{i}}-D_{x_{i},x_{j}}).\]
Combining both directions we get:
\[\sum_{x_{j}\in X}\mathbb{E}\mathcal{L}_{margin}^{f}(x_{i},x_{j})=(2p-1)\cdot\left(\sum_{Y_{j}\in A}D_{x_{i},x_{j}}-\sum_{Y_{j}\notin A}D_{x_{i},x_{j}}\right)\]
Since: $p>0.5$ and $|D_{x_{i},x_{j}}-\beta_{x_{i}}|\leq\alpha$, the minimal value is achieved whenever $D_{x_{i},x_{j}}=0$, $D_{x_{i},x_{k}}=2\alpha$ and $\beta_{x_{i}}=\alpha$, for every $x_{i},x_{j}\in A_r,\quad x_{k}\notin A_r$.
\end{proof}

\subsection*{B:\quad  Easy Positive Sampling in noisy environment}
In this subsection we analyse the EPS method from the theoretical prospective, using the framework defined in Section 4. We use the same notions as in sections 3 and 4.

Define: $\Phi(y_{i},y_{j})=\begin{cases}
1 & y_{i}=y_{j}\wedge D_{x_{i},x_{j}}=min\{D_{x_{i},x_{k}}|\,y_{k}=y_{i}\}\\
0 & else
\end{cases}$. Then, the easy positive sampling loss can be defined by: \[\frac{1}{n}\sum_{1\leq i,j,k\leq n}\Phi(y_{i},y_{j})\cdot\mathcal{L}{}_{trip}^{f}(x_{i},x_{j},x_{k})\]
for the triplet loss and \[\frac{1}{n}\sum_{1\leq i,j\leq n}(\Phi(y_{i},y_{j})\cdot\mathcal{L}{}_{margin}^{f,\beta}(x_{i},x_{j}))+1_{y_{i}\neq y_{j}}\mathcal{L}{}_{margin}^{f,\beta}(x_{i},x_{j})\]
for the margin loss.

In the noisy environment stochastic case, using section 4 notions, $\Phi$ becomes a random variable:\[\bar{\Phi}(Y_{i},Y_{j})=\begin{cases}
1 & Y_{i}=Y_{j}\wedge\forall t\,((D_{x_{i,},x_{t}}<D_{x_{i},x_{j}})\rightarrow Y_{t}\neq Y_{i})\\
0 & else
\end{cases}\] Therefore, the triplet loss with EPS in the noisy environment case, becomes:
\[
 \mathbb{E}\mathcal{L}_{EPStrip}^{f}(x_{i},x_{j},x_{k})=\mathbb{E}\left(\bar{\Phi}(Y_{i},Y_{j})\cdot\bar{\delta}_{Y_{i},Y_{j}}\cdot(1-\bar{\delta}_{Y_{i},Y_{k}})\right)\cdot\left(D_{x_{i},x_{j}}^{f}-D_{x_{i},x_{k}}^{f}+\alpha\right)_{+}\\
\]
and for the margin loss with EPS we have: 
\[
 \mathbb{E}\mathcal{L}_{EPSmargin}^{f}(x_{i},x_{j})=\mathbb{E}(\bar{\Phi}(Y_{i},Y_{j})\cdot\bar{\delta}_{Y_{i},Y_{j}})\cdot(D_{x_{i},x_{j}}^{f}-\beta_{x_{i}}+\alpha)_{+}+\mathbb{E}(1-\bar{\delta}_{Y_{i},Y_{j}})\cdot(\beta_{x_{i}}-D_{x_{i},x_{j}}^{f}+\alpha)_{+}
\]

As in section 4.2 we assume that $Y = \{Y_1,..,Y_n\}$ is a set of independent binary random variables. Let $A_1,..,A_t\subset X$, $0.5<p<1$ such that: $|A_j| = \frac{n}{t}$ and \[\mathbb{P}(Y_i=k)= \begin{cases}
p & x_i\in A_k\\
q':=\frac{1-p}{t-1} & x_i\notin A_k
\end{cases} \]
For simplicity we  assume that every  $1\leq i \leq \frac{n}{t}$ satisfies $x_{\frac{n\cdot i}{t}+1},..,x_{\frac{n\cdot i}{t}+t}\in A_i$

We prove first that the minimal embedding with respect to both
losses does not satisfy the class collapsing property. 
Let  $f_{1}$ be an embedding function such that: \[D_{x_{i},x_{j}}^{f_{1}}=\begin{cases}
0 & (\exists r)(x_{i},x_{j}\in A_{r})\\
\alpha & else
\end{cases}\] and $f_{2}$ an embedding such that: \[D_{x_{1},x_{2}}^{f_{1}}=\begin{cases}
0 & (\exists r)(x_{i},x_{j}\in A_{r})\wedge\sim((i<\frac{t}{2n}\wedge j>\frac{t}{2n})\vee(i>\frac{t}{2n}\wedge j<\frac{t}{2n})\\
\alpha & else
\end{cases}\]  $f_{1}$ represent the case of class collapsing, where $f_{2}$
represent the case where there are two modalities for the first class. In
order to show that the minimal embedding does not satisfy the class
collapsing property it suffices to prove that \[\frac{1}{n}\sum_{1\leq i,j,k\leq n}\mathbb{E}\mathcal{L}_{EPStrip}^{f_{2}}(x_{i},x_{j},x_{k})<\frac{1}{n}\sum_{1\leq i,j,k\leq n}\mathbb{E}\mathcal{L}_{EPStrip}^{f_{1}}(x_{i},x_{j},x_{k})\]
and \[\frac{1}{n}\sum_{1\leq i,j\leq n}\mathbb{E}\mathcal{L}_{EPSmargin}^{f_{2}}(x_{i},x_{j})<\frac{1}{n}\sum_{1\leq i,j\leq n}\mathbb{E}\mathcal{L}_{EPSmargin}^{f_{1}}(x_{i},x_{j}).\]
\textbf{Remark:} For both losses the definition requires a strict order between
the elements, therefore by distance zero, we meant infinitesimal close,
the order between the elements inside the sub-clusters is random, and
element between set $A_1$ are closer then set $A_1^c$ in both embeddings. For
simplification we neglect this infinitesimal constants in the proofs.
\theoremstyle{claim}
\begin{claim}{}
There exists $M$ such that if $n\geq M$, then: \[\frac{1}{n}\sum_{1\leq i,j,k\leq n}\mathbb{E}\mathcal{L}_{EPStrip}^{f_{2}}(x_{i},x_{j},x_{k})<\frac{1}{n}\sum_{1\leq i,j,k\leq n}\mathbb{E}\mathcal{L}_{EPStrip}^{f_{1}}(x_{i},x_{j},x_{k})\]
\end{claim}
\begin{proof}
Fix $x_{1}$, WOLOG we may assume in both embeddings that $D_{x_{1},x_{i}}^{f_{j}}<D_{x_{1},x_{k}}^{f_{j}}$
for every $j\in\{1,2\}$ and $1\leq i<k\leq n$. It suffices to prove
that \[\frac{1}{n}\sum_{1\leq j,k\leq n}(\mathbb{E}\mathcal{L}_{EPStrip}^{f_{1}}(x_{1},x_{j},x_{k})-\mathbb{E}\mathcal{L}_{EPStrip}^{f_{2}}(x_{1},x_{j},x_{k}))>0\] Let $q=(1-p)$, observe that \[\mathbb{P}(\bigwedge_{1\leq t<j}Y_{i}\neq Y_{t})=p^{m+1}\cdot q^{j-2-m}+p^{j-2-m}q^{j+1}\leq2p^{j-1}\]
where $m=|\{t\,|\,t\leq j,\,\,Y_{t}\in A_1\}|$ . Thus if $j\geq\frac{n}{2t}$,
we have
\[
\mathbb{E}\mathcal{L}_{EPStrip}^{f_{2}}(x_{1},x_{j},x_{k})  \leq \mathbb{P}(\bigwedge_{1\leq t<j}Y_{i}\neq Y_{t})\cdot2\cdot\alpha\leq4\cdot\alpha p^{j-1}.\] Therefore, 
\[
\frac{1}{n}\sum_{j>\frac{n}{2t},1\leq k\leq n}\mathbb{E}\mathcal{L}_{EPStrip}^{f_{2}}(x_{1},x_{j},x_{k})  \leq\sum_{j>\frac{n}{2t}}4\cdot\alpha p^{j}=4\cdot n\cdot\alpha p^{\frac{n}{2t}}\cdot\sum_{j=0}^{\frac{n(2t-1)}{2t}}p^{j}=4\cdot\alpha p^{\frac{n}{2t}}\cdot\frac{1-q^{n(2t-1)/2t}}{1-q}\stackrel{n\rightarrow\infty}{\rightarrow}0
\]
For $j\leq\frac{n}{2t}$ and $k\leq\frac{n}{2t}$ of $k>\frac{n}{t}$,
we have $\mathbb{E}\mathcal{L}_{EPStrip}^{f_{1}}(x_{1},x_{j},x_{k})=\mathbb{E}\mathcal{L}_{EPStrip}^{f_{2}}(x_{1},x_{j},x_{k})$.
Hence, the only case left is $j\leq\frac{n}{2t}$ and $\frac{n}{2t}<k\leq\frac{n}{t}$.
In this case: $\mathbb{E}\mathcal{L}_{EPStrip}^{f_{2}}(x_{1},x_{j},x_{k})=0$,
where \[\mathbb{E}\mathcal{L}_{EPStrip}^{f_{1}}(x_{1},x_{j},x_{k})=(p^{2}\cdot q^{j-1}+q^{2}\cdot p^{j-1})\cdot\alpha\geq q^{j+1}\alpha\] and we get:
\begin{align*}
 & \frac{1}{n}\cdot\sum_{j\leq\frac{n}{2t},\frac{n}{2t}\leq k\leq\frac{n}{t}}\mathbb{E}\mathcal{L}_{EPStrip}^{f_{1}}(x_{1},x_{j},x_{k})-\mathbb{E}\mathcal{L}_{EPStrip}^{f_{2}}(x_{1},x_{j},x_{k})\geq\\
 & \alpha\cdot q^{2}\cdot\sum_{j=0}^{\frac{n}{2t}}q^{i}=\alpha\cdot q^{2}\cdot\frac{1-q^{n/2t}}{1-q}\stackrel{n\rightarrow\infty}{\rightarrow}\alpha q^{2}\cdot\frac{1}{1-q}
\end{align*}
Choosing $M$ such that \[\alpha\cdot q^{2}\cdot\frac{1-q^{M/2t}}{1-q}>4\cdot\alpha p^{\frac{M}{4}}\cdot\frac{1-q^{M(2t-1)/2t}}{1-q}\]
will satisfy that for every $n>M$: \[\frac{1}{n}\sum_{1\leq i,j,k\leq n}\mathbb{E}\mathcal{L}_{EPStrip}^{f_{2}}(x_{i},x_{j},x_{k})<\frac{1}{n}\sum_{1\leq i,j,k\leq n}\mathbb{E}\mathcal{L}_{EPStrip}^{f_{1}}(x_{i},x_{j},x_{k})\]
\end{proof}
\begin{claim}{}
There exists $M$ such that if $n\geq M$ then: \[\frac{1}{n}\sum_{1\leq i,j\leq n}\mathbb{E}\mathcal{L}_{EPSmargin}^{f_{2}}(x_{i},x_{j})<\frac{1}{n}\sum_{1\leq i,j\leq n}\mathbb{E}\mathcal{L}_{EPSmargin}^{f_{1}}(x_{i},x_{j})\]
\end{claim}
\begin{proof}
For every $1\leq j\leq\frac{n}{2t}$ or $\frac{n}{t}<j\leq n$
we have: \[\mathbb{E}\mathcal{L}_{EPSmargin}^{f_{1}}(x_{i},x_{j})=\mathbb{E}\mathcal{L}_{EPSmargin}^{f_{1}}(x_{i},x_{j})\]
For $\frac{n}{2t}<j\leq\frac{n}{2}$: \[\mathbb{E}\mathcal{L}_{EPSmargin}^{f_{2}}(x_{i},x_{j})=2\cdot p\cdot q\cdot\beta_{x_{i}}+(p^{2}q^{j-2}+q^{2}p^{j-2})\cdot(2\cdot\alpha-\beta_{x_{i}})\]
while: \[\mathbb{E}\mathcal{L}_{EPSmargin}^{f_{2}}(x_{i},x_{j})=2\cdot p\cdot q\cdot(\beta_{x_{i}}+\alpha)+(p^{2}q^{j-2}+q^{2}p^{j-2})\cdot(\alpha-\beta_{x_{i}})\]
Since $j>\frac{n}{2t}$ the second therm tend to zero. Therefore,
taking M such that \[2qp>(p^{2}q^{\frac{M}{2t}-2}+q^{2}p^{\frac{M}{2t}-2})\]
will satisfy that for each $n\geq M$ \[\frac{1}{n}\sum_{1\leq i,j\leq n}\mathbb{E}\mathcal{L}_{EPSmargin}^{f_{2}}(x_{i},x_{j})<\frac{1}{n}\sum_{1\leq i,j\leq n}\mathbb{E}\mathcal{L}_{EPSmargin}^{f_{1}}(x_{i},x_{j})\]
\end{proof}
In the previous two claims we prove that the class collapsing solution is
not minimal with respect to both the $EPSmargin$ and the $EPStriplet$.
In the following claims we prove that not only it is not the minimal solution, looking
locally on the direct effect of the EPS losses on a sample which is not one of the closest elements to to the anchor. We prove that the optimal solution in this case is an embedding in which the distance between the sample to the anchor is equal to the
margin hyperparameter.
\begin{claim}{}
Let $f$ be an embedding. For every $i$, let $i_{1},..,i_{n}$
be such that $D_{x_{i},x_{i_{1}}}^{f}<D_{x_{i},x_{i_{2}}}^{f}<...<D_{x_{i},x_{n}}^{f}$,
Then there exists $M$ such that for every $j>M$ the minimal embedding
for $\mathbb{E}\mathcal{L}_{EPSmargin}^{f}(x_{i},x_{j})$ is achived whenever
$D_{x_{i},x_{j}}^{f}=\beta_{x_{i}}+\alpha$.
\end{claim}
\begin{proof}
Fix $x_{1}$, as in the previous claims we will assume: \[D_{x_{1},x_{1}}^{f}<D_{x_{1},x_{2}}^{f}<...<D_{x_{1},x_{n}}^{f}\]
As was prove in in Claim 1 $\mathbb{P}(\bigwedge_{1\leq t<j}Y_{i}\neq Y_{t})\leq4p^{j}$,
thus \[\mathbb{E}(\bar{\Phi}(Y_{i},Y_{j})\cdot\bar{\delta}_{Y_{i},Y_{j}})\leq \mathbb{P}(\bigwedge_{1\leq t<j}Y_{i}\neq Y_{t})\leq4p^{j-1}\stackrel{j\rightarrow\infty}{\rightarrow}0\]
Since the minimal solution for \[\mathbb{E}(\bar{\Phi}(Y_{i},Y_{j})\cdot\bar{\delta}_{Y_{i},Y_{j}})\cdot(D_{x_{i},x_{j}}^{f}-\beta_{x_{i}}+\alpha)_{+}+\mathbb{E}(1-\bar{\delta}_{Y_{i},Y_{j}})\cdot(\beta_{x_{i}}-D_{x_{i},x_{j}}^{f}+\alpha)_{+}\]
 satisfies $|\beta_{x_{i}}-D_{x_{i},x_{j}}^{f}|\leq\alpha$, we have:
 \begin{align*}
 & \mathbb{E}\mathcal{L}_{EPSmargin}^{f_{1}}(x_{1},x_{j})=\alpha\cdot(\mathbb{E}(\bar{\Phi}(Y_{1},Y_{j})\cdot\bar{\delta}_{Y_{1},Y_{j}})+\mathbb{E}(1-\bar{\delta}_{Y_{1},Y_{j}}))+\\
 & (D_{x_{1},x_{j}}^{f}-\beta_{x_{1}})\cdot(\mathbb{E}(\bar{\Phi}(Y_{1},Y_{j})\cdot\bar{\delta}_{Y_{1},Y_{j}})-\mathbb{E}(1-\bar{\delta}_{Y_{i},Y_{j}}))
\end{align*}
Since $\mathbb{E}(1-\bar{\delta}_{Y_{1},Y_{j}})\geq2pq$, there
exists $M$ such every $j>M$ satisfies \[(\mathbb{E}(\bar{\Phi}(Y_{1},Y_{j})\cdot\bar{\delta}_{Y_{1},Y_{j}})-\mathbb{E}(1-\bar{\delta}_{Y_{i},Y_{j}}))<0\] Therefore the minimal value is achieved whenever $D_{x_{1},x_{j}}^{f}=\alpha+\beta_{x_{1}}$.
\end{proof}

The proof in the EPStriplet loss case is similar.
\begin{claim}{}
Let $f$ be an embedding. For every $i$, let $i_{1},,..,i_{n}$
be such that $D_{x_{i},x_{i_{2}}}^{f}<...<D_{x_{i},x_{n}}^{f}$.
Then there exists $M$ such that for every $j>M$ the minimal embedding
for: \[\mathbb{E}\mathcal{L}_{EPStrip}^{f}(x_{i},x_{t},x_{t+j})+\mathbb{E}\mathcal{L}_{EPStrip}^{f}(x_{i},x_{t+j},x_{t})\] is achieved whenever $D_{x_{i},x_{t+j}}=D_{x_{i},x_{t}}+\alpha$.
\end{claim}
\begin{proof}
Define $K(Y_{i},Y_{i},Y_{k}):=\mathbb{E}\left(\bar{\Phi}(Y_{i},Y_{j})\cdot\bar{\delta}_{Y_{i},Y_{j}}\cdot(1-\bar{\delta}_{Y_{i},Y_{k}})\right)$.
Fixing $x_{1}$, assuming $D_{x_{1},x_{1}}^{f}<D_{x_{1},x_{2}}^{f}<...<D_{x_{1},x_{n}}^{f}$,  We have:
\begin{align*}
 & \mathbb{E}\mathcal{L}_{EPStrip}^{f}(x_{1},x_{t},x_{t+j})+\mathbb{E}\mathcal{L}_{EPStrip}^{f}(x_{1},x_{t+j},x_{t})=K(Y_{1},Y_{t},Y_{t+j})\cdot\left(D_{x_{1},x_{t}}^{f}-D_{x_{1},x_{t+j}}^{f}+\alpha\right)_{+}\\
 & +K(Y_{1},Y_{t+j},Y_{t})\cdot\left(D_{x_{1},x_{t+j}}^{f}-D_{x_{1},t}^{f}+\alpha\right)_{+}
\end{align*}
As in the previous claim, the minimal value is achieved whenever
$|D_{x_{1},x_{t+j}}^{f}-D_{x_{1},x_{t}}^{f}|\leq\alpha$ in this case: \begin{align*}
 & \mathbb{E}\mathcal{L}_{EPStrip}^{f}(x_{1},x_{t},x_{t+j})+\mathbb{E}\mathcal{L}_{EPStrip}^{f}(x_{1},x_{t+j},x_{t})=\alpha\cdot(K(Y_{1},Y_{t},Y_{t+j})+K(Y_{1},Y_{t+j},Y_{t}))Y_{t}))+\\
 & (D_{x_{1},x_{t}}^{f}-D_{x_{1},x_{t+j}}^{f})\cdot(K(Y_{1},Y_{t},Y_{t+j})-K(Y_{1},Y_{t+j},Y_{t}))
\end{align*}
On the one hand: $K(Y_{1},Y_{t},Y_{t+j})=(\prod_{i\in\{1,2,..,t,t+j\}}p^{t_{i}}q^{1-t_{i}})+(\prod_{i\in\{1,2,..,t,t+j\}}p^{1-t_{i}}q^{t_{i}})\geq q^{t+1}$
where $t_{k}=\begin{cases}
1 & Y_{k}\notin A\\
0 & else
\end{cases}$ for $k\in\{2,..,t-1,t+j\}$ and $t_{k}=\begin{cases}
1 & Y_{k}\in A\\
0 & else
\end{cases}$for $k\in\{1,t\}$. On the other hand $K(Y_{1}Y_{t+j},Y_{t})\geq Prob(\bigwedge_{1\leq k<t+j}Y_{1}\neq Y_{k})\leq4p^{t+j-1}$.
Taking $j$ large enough such that $q^{t+1}\leq4p^{t+j-1}$, we have:
\[(\mathbb{E}\left(\bar{\Phi}(Y_{1},Y_{t})\cdot\bar{\delta}_{Y_{1},Y_{t}}\cdot(1-\bar{\delta}_{Y_{1},Y_{t+j}})\right)-\mathbb{E}\left(\bar{\Phi}(Y_{1},Y_{t+j})\cdot\bar{\delta}_{Y_{1},Y_{t+j}}\cdot(1-\bar{\delta}_{Y_{1},Y_{t}})\right))>0\]
therefore in such case the minimum is archived whenever $D_{x_{1},x_{t+j}}^{f}=D_{x_{1},x_{t}}^{f}+\alpha$.

\end{proof}

\subsection*{C:\quad  More experiments and implementation details}

  

\subsubsection* {MNIST architecture details}
For the MNIST even/odd experiment we use a model consisting of two consecutive convolutions layer with (3,3) kernels and 32,64 (respectively) filter sizes. The two layers are followed by Relu activation and batch normalization layer, then there is a (2,2) max-pooling follows by 2 dense layers with 128 and 2 neurons respectively.


\begin{table}[!t]
\centering

\tablestyle{3.75pt}{1.1}
\begin{tabular}{l|l|c|x{30}}
\multicolumn{1}{c|}{\multirow{2}{*}{dataset}} & \multicolumn{1}{c|}{\multirow{2}{*}{model}} & \multicolumn{1}{c|}{Without EPS} & \multicolumn{1}{c}{With EPS} \\
& & std & std  \\
\shline
cars196 & Margin & 0.17 & 0.27 \\
cars196 & MS & 0.24 & 0.29 \\
cars196 & Trip+SH & 0.20 & 0.47 \\
\hline
cub200 & Margin & 29.8 & 0.33 \\
cub200 & MS & 0.43 & 0.36 \\
cub200 & Trip+SH & 0.52 & 0.35 \\
\hline
Omniglot-letters & Margin & 0.73 & 0.58 \\
Omniglot-letters & MS & 0.52 & 0.71 \\
Omniglot-letters & Trip+SH & 0.34 & 0.61 \\
\end{tabular}

\caption{Std of Recall@1 results. Each model was trained 8 times with different random seeds.}
\label{Std_table}
\end{table}

\begin{table}[!t]
\centering
\tablestyle{3.75pt}{1.1}
\begin{tabular}{l|x{30}x{30}|x{30}x{30}|x{30}x{30}}
\multicolumn{1}{c|}{\multirow{2}{*}{}} & \multicolumn{2}{c|}{Cars196} & \multicolumn{2}{c|}{CUB200} \\
& MS & MS+EPS & MS & MS+EPS  \\
\shline
R@1 & 84.1 & \textbf{85.5} & 65.7 & \textbf{66.7}  \\
R@2 & 90.4 & \textbf{90.7} & 77.0 & \textbf{77.2}   \\
R@4 & 94.0 & \textbf{94.3} & 86.3 & \textbf{86.4}  \\
R@8 & 96.5 & \textbf{96.7} & \textbf{91.2} & 90.9   \\
\end{tabular}
\caption{Results of Multi-similarity loss with Embedding size 512 (as in ~\cite{wang2019multi}). Using EPS improve results in both cases. }
\label{MS-table}

\end{table}

\subsubsection* {Stability analysis}
Following \cite{musgrave2020metric,1911.12528}, it was important to us to have a fair comparison between all tested models. Therefore, for all the experiments we use the same framework as in~\cite{Roth2020}, with the same architecture and embedding size (128). We also did not change the default hyper-parameters in all tested methods. We run each experiment 8 times with different random seeds, the reported results are the mean of all the experiments. The std of the Recall@1 results of all experiments can be seen in Table \ref{Std_table}. In all cases the differences between the results with and without the EPS are significance.

\subsubsection* {Multi-similarity comparison}
From our experiments, the Multi-similarity loss is highly affected by the batch size. Using Resnet50 backbone, we restrict the number of batch size to 160 for all tested model, which cause to the inferior results of the multi-similarity loss comparing to other methods. For the sake of completeness we provide the results also on inception backbone with embedding size of 512 as in~\cite{wang2019multi}, and batch size of 260. As can be seen in Table \ref{MS-table}, also in these cases the results improve when using EPS instead of semi-hard sampling on the positive samples.




\end{document}